%% file: main.tex
\documentclass{article}

\usepackage{PRIMEarxiv}

\usepackage{amsthm}
\usepackage[utf8]{inputenc} 
\usepackage[T1]{fontenc}    
\usepackage{hyperref}       
\usepackage{url}            
\usepackage{booktabs}       
\usepackage{amsfonts}       
\usepackage{nicefrac}       
\usepackage{microtype}      
\usepackage{lipsum}
\usepackage{fancyhdr}       
\usepackage{graphicx}       
\graphicspath{{media/}}     

\usepackage{thm-restate}
\usepackage{xcolor}
\definecolor{DarkGreen}{rgb}{0.1,0.5,0.1}

\usepackage{bbm}
\usepackage{amsmath}  
\usepackage[english]{babel}
\usepackage{bm}

\usepackage[ruled,vlined]{algorithm2e}
\usepackage{algorithmic}

\usepackage{epstopdf,fancyhdr,amsmath,xspace,enumerate,paralist, mathtools,tabularx,hhline,tabu,forest,array,stmaryrd}
\usepackage[export]{adjustbox}
\usepackage{pgf,tikz}
\usepackage{bm}
\usepackage{subcaption}
\usepackage{mathrsfs}
\usepackage{type1cm}
\usepackage{csquotes}
\usepackage{comment}
\usepackage{multirow}
\usepackage{verbatim}
\usepackage{enumitem}
\usepackage{natbib}
\usepackage{pdflscape}
\usepackage{thm-restate}
\usepackage{soul}
\usepackage{stmaryrd}
\usepackage{environ}
\usepackage{wrapfig}
\usepackage{float}

\usepackage[on]{editing}

\tikzstyle{mybox} = [draw=gray, fill=gray!20, very thick,
rectangle, rounded corners, inner xsep=3pt, inner ysep=7pt]

\usepackage{pgfplots}
\DeclareUnicodeCharacter{2212}{−}
\usepgfplotslibrary{groupplots,dateplot}
\usetikzlibrary{patterns,shapes.arrows}
\pgfplotsset{compat=newest}

\usetikzlibrary{positioning,arrows,shapes,shapes.arrows,arrows.meta,trees,shapes.misc,shapes.geometric,decorations.pathreplacing,automata,patterns,fadings}
\tikzset{%
  >={Latex[width=1.5mm,length=1.8mm]},
  vertex/.style={draw,circle,inner sep=0mm,semithick,minimum width=4.7mm},
  bvertex/.style={draw,circle,inner sep=0mm,semithick,minimum width=5.2mm},
  textnode/.style={draw,ellipse,inner sep=0.5mm,semithick,minimum width=5.2mm},
  point/.style = {circle, draw, inner sep=0.04cm,fill,node contents={}},
  uvertex/.style={outer sep=0},
  box/.style={rectangle,minimum size=1cm},
  bidir/.style={<->,dashed,semithick,line width=0.25mm},
  dir/.style={->, line width=0.25mm},
  unobsdir/.style={->,dashed,semithick,line width=0.25mm},
  regime/.style={shape=rectangle,fill=black,inner sep=0pt,outer sep=1mm,minimum size=3pt,draw},
  node distance=1cm,
  font=\small\sffamily,
}

\pgfdeclarefading{fading}
 {\tikz
    \fill[bottom color=transparent!0,top color=transparent!70] (0,0) rectangle (100bp,100bp);}

\pgfdeclarelayer{back}
\pgfsetlayers{back,main}

\definecolor{lightred}{RGB}{255, 150, 141}
\definecolor{lightblue}{RGB}{86, 193, 255}
\definecolor{lightgreen}{RGB}{115, 253, 134}

\definecolor{betterred}{RGB}{255, 10, 78}
\definecolor{betterblue}{RGB}{0, 162, 255}
\definecolor{bettergreen}{RGB}{22, 231, 207}

\definecolor{darkred}{RGB}{228,26,28}
\definecolor{darkblue}{RGB}{55,126,184}
\definecolor{darkgreen}{RGB}{77,175,74}
\definecolor{darkyellow}{RGB}{255,127,0}

\usepackage{cleveref}
\Crefname{equation}{Eq.}{Eqs.}
\Crefname{figure}{Fig.}{Figs.}
\Crefname{tabular}{Tab.}{Tabs.}
\Crefname{theorem}{Thm.}{Thms.}
\Crefname{lemma}{Lem.}{Lems.}
\Crefname{proposition}{Prop.}{Props.}
\Crefname{definition}{Def.}{Defs.}
\Crefname{algorithm}{Algorithm }{Algorithms }
\Crefname{corollary}{Corol.}{Corols.}
\Crefname{section}{Sec.}{Secs.}
\Crefname{table}{Table}{Tables}
\theoremstyle{plain}
\newtheorem{theorem}{Theorem}[section]

\theoremstyle{definition}
\newtheorem{definition}[theorem]{Definition}

\theoremstyle{remark}

\newcommand{\Braces}[1]{\left\{ #1\right\}}

\newcommand{\Parens}[1]{\left(#1\right)}

\newcommand{\angles}[2][]{#1\langle#2 #1\rangle}

\newcommand{\tuple}[2][]{\angles[#1]{#2}}

\makeatletter
\newcommand{\xdashleftrightarrow}[2][]{\ext@arrow 3359\leftrightarrowfill@@{#1}{#2}}
\def\rightarrowfill@@{\arrowfill@@\relax\relbar\rightarrow}
\def\leftarrowfill@@{\arrowfill@@\leftarrow\relbar\relax}
\def\leftrightarrowfill@@{\arrowfill@@\leftarrow\relbar\rightarrow}
\def\arrowfill@@#1#2#3#4{%
  $\m@th\thickmuskip0mu\medmuskip\thickmuskip\thinmuskip\thickmuskip
   \relax#4#1
   \xleaders\hbox{$#4#2$}\hfill
   #3$%
}
\makeatother

\newcommand{\I}{\mathbb{I}}

\newcommand{\VV}{\bm{V}}

\newcommand{\XX}{\bm{X}}
\newcommand{\xx}{\bm{x}}

\newcommand{\D}{\Omega}

\newcommand{\G}{\mathcal{G}}

\newcommand{\Pa}{\mathrm{Pa}}
\newcommand{\pa}{\mathrm{pa}}
\newcommand{\PA}{\mathrm{PA}}

\newcommand{\doo}{\mathrm{do}}

\def\*#1{\boldsymbol{#1}}
\def\1#1{\mathcal{#1}}
\def\2#1{\mathcal{#1}}
\def\3#1{\mathbb{#1}}
\def\4#1{\bar{\*#1}}

\title{Partial Identification Approach to \\Counterfactual Fairness Assessment}

\author{
  Saeyoung Rho \\
  Department of Computer Science \\
  Columbia University \\
  New York, NY\\
  \texttt{s.rho@columbia.edu} \\
  \And
  Junzhe Zhang \\
  Department of Electrical Engineering and Computer Science \\
  Syracuse University\\
  Syracuse, NY\\
  \texttt{jzhan403@syr.edu} \\
  \AND
  Elias Bareinboim \\
  Department of Computer Science \\
  Columbia University \\
  New York, NY\\
  \texttt{eb@cs.columbia.edu} \\
}

\begin{document}
\maketitle

\input{section0}


\input{section1}
\input{section2}
\input{section3}

\input{section4}
\input{section5}

\input{section6}

\input{section7}

\bibliographystyle{unsrtnat}
\bibliography{references}


\newpage

\newgeometry{onecolumn} 
\appendix

\input{appendix}

\end{document}

%% file: section0.tex
\begin{abstract}
The wide adoption of AI decision-making systems in critical domains such as criminal justice, loan approval, and hiring processes has heightened concerns about algorithmic fairness.
As we often only have access to the output of algorithms without insights into their internal mechanisms, it was natural to examine how decisions would alter when auxiliary sensitive attributes (such as race) change. This led the research community to come up with counterfactual fairness measures, but how to \emph{evaluate} the measure from available data remains a challenging task. In many practical applications, the target counterfactual measure is not identifiable, i.e., it cannot be uniquely determined from the combination of quantitative data and qualitative knowledge. This paper addresses this challenge using partial identification, which derives informative bounds over counterfactual fairness measures from observational data. We introduce a Bayesian approach to bound unknown counterfactual fairness measures with high confidence. We demonstrate our algorithm on the COMPAS dataset, examining fairness in recidivism risk scores with respect to race, age, and sex.
Our results reveal a positive (spurious) effect on the COMPAS score when changing race to African-American (from all others) and a negative (direct causal) effect when transitioning from young to old age.

\end{abstract}

%% file: section1.tex
\section{Introduction}
\label{sec:Introduction}

Algorithmic decision-making systems have become an integral part of our lives, significantly influencing critical aspects of our society where fairness and equity are paramount. 
From determining access to healthcare resources to shaping lending practices and criminal justice outcomes, these algorithms have the ability to either uphold or compromise justice \citep{lee2019procedural}.
While we cannot halt the use of those algorithms, we have both the opportunity and the obligation to equip ourselves to examine their behaviors and evaluate how fair these algorithms are \citep{mehrabi2021survey, mitchell2021algorithmic}.

One of the primary challenges in assessing fairness is our limited understanding of the inner workings of these algorithms, as they function as opaque ``black boxes'' \citep{adebayo2016fairml,saleiro2018aequitas}. Typically, we only have access to the outputs of the algorithm, such as recommendations or decisions, while the inputs and the intricate processes concealed within these black boxes remain hidden \citep{bandy2021problematic}.
Consequently, most fairness metrics discussed so far are about retrospectively evaluating the parity of fairness of outcomes---such as False Negative Rate (FNR), False Positive Rate (FPR), and Positive Predicted Value (PPV, the positive predictions that turn out to be indeed positive)---across groups by race or gender. 
However, it is known that these parities cannot be achieved unless we assume that the true distribution also achieves parity \citep{chouldechova2017fair, berk2021fairness}. Moreover, the observational criteria often fail to capture indirect discrimination through proxies \citep{kilbertus2017avoiding}.

What if we assume access to a detailed causal model that generates the data? By manipulating the model's causal mechanisms, one could simulate thought experiments to explore hypothetical scenarios where an attribute believed by society to be irrelevant is altered. For instance, when deciding on a loan approval, one could imagine a hypothetical scenario where the applicant's race is altered: i.e., asking a counterfactual query, ``What if an individual's race were not white but black while holding all the other attributes unchanged?'' \citep{kusner2017counterfactual}.
If the algorithm's outputs change in this alternative scenario, we may conclude that the algorithm was unfair. Based on this intuition, researchers have put forth various definitions/measurements of counterfactual fairness over the past few years \citep{kusner2017counterfactual, chiappa2019path, zhang2018fairness, nabi2018fair,zhang2017anti,kilbertus2017avoiding}.
These measures take into consideration the underlying disparity in true distribution across groups if it exists, and also well-represent primary anti-discrimination frameworks applied in legal systems throughout the US and the EU: disparate impact and disparate treatment \citep{barocas2016big}.

Despite their transparency and intuitiveness, several challenges exist in evaluating these counterfactual fairness measures. In many applications, the detailed causal model is often not fully known; instead, the learner/investigator can access data generated by the algorithm through passive observation or auditing.
First of all, it necessitates the construction of a causal diagram \citep{pearl:95} encoding qualitative causal relationships among variables in the data-generating model. Learning valid causal diagrams is an active area of research studied under the rubrics of causal discovery \citep{pearl:2k, spirtes2000causation,petersen:etal06}. Even if we can access an accurate causal diagram, obtaining counterfactual probabilities based on observational data is still challenging (Section 4.4 of \cite{plecko2022causal}). The causal knowledge encoded in the diagram is often insufficient to allow target counterfactual probabilities to be identifiable, i.e., they cannot be uniquely discernible from available data \citep[Ch.~3]{pearl:2k}.
Some algorithms have been suggested to identify counterfactual fairness measures. Still, they are restricted to identifiable cases only and/or work under strong parametric assumptions such as linearity in the model \citep{nabi2018fair, wu2019counterfactual,zhang2016causal}.

Our goal in this paper is to address these challenges. We introduce a novel partial identification algorithm that could bound any counterfactual fairness measure from observational data. Our algorithm can be applied to any categorical data, without assuming strong parametric functional assumption, and derive informative bounds with theoretical guarantees. Also, we demonstrate the algorithm on the COMPAS (Correctional Offender Management Profiling for Alternative Sanctions) recidivism risk score dataset \citep{propublica}.
Through this case study, we explain the causal mechanisms behind the discrimination embedded within the COMPAS algorithm on the basis of race, sex, and age. Our research contributes to the ongoing efforts to correctly evaluate the (un)fairness of the algorithms while access to the full picture is limited. 
More specifically, our contributions are summarized as follows:
\begin{enumerate}
    \item We develop a novel partial identification algorithm for bounding counterfactual fairness measures from observational data, with a theoretical guarantee;
    \item We evaluate our algorithm using both simulations and real-world datasets.
    \item Using our algorithm, we analyze behaviors of the COMPAS algorithm and reveal new explanations for its disparate impact on defendants with minority backgrounds.
\end{enumerate}
Given space constraints, additional details of the experimental setups are provided in Appendix \ref{app:exp}.

%% file: section2.tex
\section{Preliminaries and Related Work}
\label{sec:Background}

In this section, we introduce the necessary definitions and theorems that we use throughout the paper. The basic notations are as follows. We use capital letters to denote variables ($\*X$), small letters for their values ($x$), and $\D_X$ for their domains. For an arbitrary set $\D$, let $|\D|$ be its cardinality. The distribution over variables $\*X$ is denoted by $P(\*X)$.
For convenience, we consistently use $P(\*x)$ as a shorthand for the probability $P(\*X = \*x)$.
Finally, the indicator function $\I_{\*X = \*x}$ returns $1$ 
if an event $\*X = \*x$ holds; otherwise, $\I_{\*X = \*x}$ is equal to $0$.

\textbf{Structural Causal Models. }The basic semantical framework of our analysis rests on \textit{structural causal models} (SCMs) \citep{pearl2009causality}. An SCM $\1M$ is a tuple $\tuple{\*V, \*U, \2F, P(\*U)}$ where $\*V$ is a set of endogenous variables and $\*U$ is a set of exogenous variables. $\2F$ is a set of functions where each $f_V \in \2F$ decides values of an endogenous variable $V \in \*V$, taking a combination of other variables in the system as an argument. That is, $v \leftarrow f_{V}(\pa_V, u_V), \*\PA_V \subseteq \*V, U_V \subseteq \*U$. Exogenous variables $U \in \*U$ are mutually independent, values of which are drawn from the exogenous distribution $P(\*U)$. Naturally, $\1M$ induces a joint distribution $P(\*V)$ over endogenous variables $\*V$, called the \emph{observational distribution}. Each SCM $\1M$ is also associated with a causal diagram $\G$ (e.g., \Cref{fig.standard}), which is a directed acyclic graph (DAG) where solid nodes represent endogenous variables $\*V$, empty nodes represent exogenous variables $\*U$, and arrows represent the arguments $\PA_V, U_V$ of each structural function $f_{V}$.

Intervention on an arbitrary subset $\XX \subseteq \VV$, denoted by $\doo(\xx)$, is an operation where values of $\XX$ are set to constants $\xx$, regardless of how they are ordinarily determined. For an SCM $\1M$, let $\1M_{\*x}$ denote a submodel of $\1M$ induced by intervention $\doo(\xx)$. For any subset $\*Y \subseteq \*V$, the \emph{potential response} $\*Y_{\*x}(\*u)$ is defined as the solution of $\*Y$ in the submodel $M_{\*x}$ given $\*U = \*u$. Drawing values of exogenous variables $\*U$ following the distribution $P(\*U)$ induces a \emph{counterfactual variable} $\*Y_{\*x}$. Specifically, the event $\*Y_{\*x} = \*y$ (for short, $\*y_{\*x}$) can be read as ``$\*Y$ would be $\*y$ had $\*X$ been $\*x$''.
For subsets $\*Y, \dots, \*Z$, $\*X, \dots, \*W \subseteq \*V$, the distribution over counterfactuals $\*Y_{\*x}, \dots, \*Z_{\*w}$ is defined as:
\begin{align}
\! P \left (\*y_{\*x}, \dots, \*z_{\*w} \right) = \int_{\D_{\*U}} \I_{\*Y_{\*x}(\*u) = \*y, \dots, \*Z_{\*w}(\*u) = \*z} dP(\*u). \label{eq:cfd}
\end{align}
Distributions of the form $P(\*Y_{\*x})$ are called \emph{interventional distributions}; when $\*X = \emptyset$, $P(\*Y)$ coincides with the \emph{observational distribution}. Throughout this paper, we assume that domains of endogenous variables $\*V$ are discrete and finite; while exogenous variables $\*U$ could take values in any (continuous) domains. For a more detailed survey on SCMs, see \citep{pearl2009causality}.

\begin{figure}[h]
    \hfill
	\begin{subfigure}{0.4\linewidth}\centering
 		\begin{tikzpicture}
			\node[vertex] (A) at (0, 0) {A};
			\node[vertex] (Y) at (3, 0) {Y};
			\node[vertex] (W) at (1.5, -1.2) {W};
			\node[vertex] (Z) at (1.5, 1.2) {Z};

			\draw[dir] (A) -- (Y);
			\draw[dir] (A) -- (W);
                \draw[dir] (W) -- (Y);
			\draw[dir] (Z) -- (Y);
			\draw[dir] (Z) -- (A);
		\end{tikzpicture}
            \caption{}
            \label{fig.standard_a}
        \end{subfigure}\hfill
	\begin{subfigure}{0.4\linewidth}\centering
 		\begin{tikzpicture}
			\def\outerr{3.2}
			\def\innerr{3}

			\node[vertex] (A) at (0, 0) {A};
			\node[vertex] (Y) at (3, 0) {Y};
			\node[vertex] (W) at (1.5, -1.2) {W};
			\node[vertex] (Z) at (1.5, 1.2) {Z};
			\node[uvertex] (U1) at (1.5, 0.4) {$U_1$};
			\node[uvertex] (U2) at (1.5, -0.4) {$U_2$};

			\draw[dir] (Z) -- (Y);
			\draw[dir] (A) -- (Y);
                \draw[dir] (W) -- (Y);
			\draw[dir] (Z) -- (A);
			\draw[dir] (A) -- (W);
			\draw[unobsdir] (U1) -- (A);
			\draw[unobsdir] (U1) -- (Z);
			\draw[unobsdir] (U1) -- (Y);
			\draw[unobsdir] (U2) -- (A);
			\draw[unobsdir] (U2) -- (W);
			\draw[unobsdir] (U2) -- (Y);
            \end{tikzpicture}  
            \caption{}
            \label{fig.standard_b}
    \end{subfigure}\hfill\null
\caption{Causal diagrams representing a standard fairness model containing a protected attribute $A$ (e.g., race), an outcome $Y$ (recidivism score), a confounder $Z$ (birthplace) and a mediator $W$ (prior criminal records).}
    \label{fig.standard}
\end{figure}
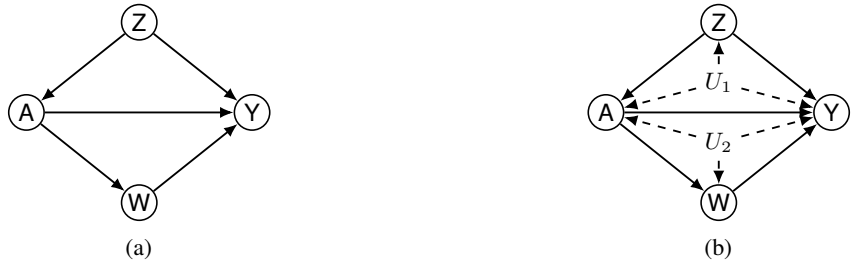

\textbf{Causal Fairness Measures. }The language of structural causality allows one to formalize and articulate concepts that are not easily defined in classic statistical theory. For instance, we can measure the impact of unfair and discriminatory practices with mathematical precision by simulating thought experiments through hypothetical interventions on structural equations. We will consistently use $A$ to stand for a protected attribute; $Y$ for the primary outcome; $Z$ for all the observed confounders affecting $A$ and $Y$, i.e., their common causes; and $W$ for the descendants of $A$ that also affect $Y$, which we call mediators.

For example, \Cref{fig.standard_a} shows a causal diagram of a standard fairness model \citep{zhang2018fairness} representing the data-generating process of some recidivism score \citep{propublica}. Here, $A$ stands for the race of the defendant; $Y$ for the predicted recidivism score; $Z$ for their birthplace; and $W$ for their prior criminal records. The judge may exhibit increased strictness towards the same action due to racial bias ($A \rightarrow W$). Also, birthplace can influence an individual's race ($Z \rightarrow A$), potentially due to historical racial segregation. Finally, all observed variables may impact the algorithmic prediction $Y$.

The counterfactual probability inspired researchers to define \emph{counterfactual fairness} measures \citep{kusner2017counterfactual}. Kusner et al. say that an algorithm is \emph{counterfactually fair} if one's prediction outcome $Y$ in the real world is similar to the $Y$ in the counterfactual world where $A$ has a different value. Let $X = Z \cup W$ be other predictors affecting outcome $Y$. Formally,
\begin{definition}[Counterfactual Effect]\label{ctf.fairness}
Under any context $X=x$ and $A=a$, we define the counterfactual effect of an intervention $A=a_0$ on $Y$, with baseline $A=a_0$, as
\begin{align}
\text{CE}_{a_0, a_1}(y|x, a) = P(y_{a_1}|x,a) - P(y_{a_0}|x, a).
\end{align}
\end{definition}
\cite{kusner2017counterfactual} declares that an algorithm is fair if $\text{CE}_{a_0, a_1}(y|x, a)=0$ for all $y$ and for any value $a_0 \neq a_1$. In the example presented in \Cref{fig.standard}, this means that the distribution of the algorithm's recidivism score prediction would not change had the race($A$) been altered, given the conditions where all other observed variables remain the same.

\citet{zhang2018fairness} suggest a set of more granular counterfactual measurements to explain observed statistical disparity in the outcome $Y$ over the protected attribute $A$ over the underlying causal pathways between them, including direct, indirect, and spurious paths.

\begin{definition}[Direct, Indirect, Spurious Effects]\label{def.ctf.family}
Given a SCM $\1M$, a counterfactual direct effect (DE), indirect effect (IE), and spurious effect (SE) of an intervention $A=a_1$ on $Y$, with baseline $A=a_0$, conditioned on $A=a$ are
\begin{align}
\text{DE}_{a_0, a_1}(y|a) &= P(y_{a_1, W_{a_0}}|a) - P(y_{a_0}|a)\\
\text{IE}_{a_0, a_1}(y|a) &= P(y_{a_0, W_{a_1}}|a) - P(y_{a_0}|a)\\
\text{SE}_{a_0, a_1}(y) &= P(y_{a_0}|a_1)- P(y_{a_0}|a_0)    
\end{align}
\end{definition}
Among quantities in the above equations, $Y_{a, W_{a'}}$ is a \textit{nested counterfactual variable} such that given a unit $\*U = \*u$, the potential outcome $Y_{a, w}(\*u)$ is equal to the solution when the input $w$ is set as the values of mediator $W_{a'}(\*u)$ under intervention $\doo(a')$. Hereinafter, we will consistently refer to distributions over counterfactual variables of the form $Y_{a, W_{a'}}$ as nested counterfactual measures. Using nested counterfactuals allows us to decompose an intervention's effect, i.e., race change, on the recidivism score into three distinct components \citep{zhang2018fairness}. First, $IE$ measures the effect caused solely by changes in prior records ($W$) induced by the intervention, achieved by altering race as an input for the function $f_{W}$.
Next, $DE$ represents the impact of an intervention directly influencing the recidivism score, removing the effect from any changes in race that may have affected prior counts. Lastly, $SE$ captures the effect mediated by changes in confounding variables, such as birthplace ($Z$). Similarly, \citet{chiappa2019path} utilized Path-Specific Effects (PSEs) \citep{pearl:01} to measure the impact of potentially discriminatory mechanisms along causal paths (i.e., one-directional paths) from the protected attribute $A$ to outcome $Y$.

The intuitive nature of counterfactual notion has inspired many other fairness definitions, counterfactual predictive parity \citep{coston2020counterfactual} and counterfactual equalized odds \citep{coston2020counterfactual,mishler2021fairness}, to name a few.
This paper will primarily focus on the counterfactual measures described in \Cref{ctf.fairness,def.ctf.family}, but our proposed algorithm applies to any counterfactual probabilities.

%% file: section3.tex
\section{Estimating Counterfactual Fairness Measures}
\label{sec:method}

As research progresses in defining more sophisticated counterfactual fairness measures, methods for identifying these measures have also been proposed.
\cite{nabi2018fair} and \cite{wu2019counterfactual} have introduced algorithms for measuring path-specific counterfactual effects, albeit under the assumption of a linear model. On the other hand, \cite{zhang2016causal} avoids assuming linearity but is still limited to identifiable cases only. In this section, we present an algorithm that improves the current state-of-the-art approaches by eliminating the linearity assumption and expanding the scope to include unidentifiable cases by adopting a Bayesian sampling approach.

The general procedure of our proposed approach is provided in \Cref{alg.fairness}, which is designed to estimate a given counterfactual fairness measure $\mu$ based on observational data $\*D$. It utilizes two sub-algorithms: (1) a causal discovery algorithm, called Fast Causal Inference (FCI) \citep{spirtes2001anytime}, which learns causal relationships from the data; and (2) a partial identification algorithm, SampleCTF, which samples the posterior counterfactual measure conditioning on the observational data and inferred causal relationships.

More specifically, Step 1 applies FCI to identify an equivalence class $\1E$ of candidate causal diagrams compatible with the observational data. This is a standard approach required in analyses based on structural causal models. Any causal discovery algorithm can replace FCI in this step, depending on the user's assumptions about the data and structure, such as noise distributions or parametric assumptions.
In Step 2, the algorithm refines the learned equivalence class by filtering out candidate causal diagrams violating the domain knowledge. 
In Steps 3-5, the algorithm enumerates through each candidate causal diagram $\1G_i$ and obtains posterior samples of the target counterfactual fairness measure conditioning on the data $\*D$ and causal knowledge $\1G_i$. 
We will discuss details of the sampling algorithm, SampleCTF, later in this section. Finally, the algorithm sorts posterior samples and returns an interval at the $(1 - \delta)\%$ confidence level for a fixed error rate $\delta \in [0, 1)$. The learner could decide the error rate $\delta$ based on the goal of the analysis. When $\delta=  0$, the return interval converges to the optimal bound guaranteed to contain the target counterfactual measure when the number of observed samples increases \citep{manski1990nonparametric,chickering1996clinician,zhang2022partial}.

\begin{algorithm}[t]
  \caption{\textsc{Identifying Fairness Measure}: IDFair($\*D, \mu, \delta$)}
  \label{alg.fairness}
\begin{algorithmic}[1]
 \REQUIRE $\*D = \left\{(V^t): t = 1, \dots, T \right \}$, fairness measure $\mu$, and error rate $\delta$
 \ENSURE A bound containing the fairness measure $\mu \in [a, b]$
 \STATE Learn an equivalence class $\1E =\text{FCI}(\*D)$ which is a finite set of causal diagrams $\1E = \Braces{\1G_1, \dots, \1G_n}$ compatible with the observational data $\*D$
 \STATE Construct a subset of causal diagram $\1E^* = \Braces{\mathcal{G}_1, \ldots, \mathcal{G}_m}$ from the equivalence class $\1E$ using the domain knowledge.
 \FOR{$\mathcal{G}_i \in \1E^*$}
    \STATE $\mu_i$ = \text{SampleCTF}($\*D, \mathcal{G}_i, \mu, N$) \#$N$ samples of fairness measure
 \ENDFOR
 \STATE $\Gamma = [\mu_{1,1}, \ldots, \mu_{1,N}, \ldots, \mu_{n,N}]$ \#concatenated $n\times N$ samples
 \STATE $\Gamma = sort(\Gamma)$ \#sorted in an increasing order
 \RETURN mean $\frac{\sum_{j=1}^{n \times N} \Gamma_{j}}{n \times N}$, $(1 - \delta)\%$ confidence interval $\left(\Gamma_{\lfloor{\delta / 2 \times n\times N}\rfloor}, \Gamma_{\lceil{(1 - \delta/2)\times n\times N}\rceil}\right)$
\end{algorithmic}
\end{algorithm}

\subsection{Partial Counterfactual Identification}\label{sec:hyperparam}
In this section, we introduce a subroutine to draw posterior samples of unknown counterfactual probabilities conditioning on the observed data and causal knowledge of the environment. Details of this subroutine, SampleCTF, are described in Algorithm \ref{alg.gibbs}. It is a Bayesian sampling algorithm that takes three elements as inputs: dataset containing \(T\) samples of \(\*V\), \(\*D = \left\{(\*V^t): t = 1, \dots, T \right \}\), where \(V\) is a set of observed variables, a causal graphical model $\mathcal{G}$, and a target fairness measure $\mu$.
The causal graphical model \(\mathcal{G}\) provides quantitative information about a set of unobserved variables \(U\) and the structure of \(f_V\), a function determining an observed variable \(v \in V\). Although the specifics of \(f_V\) are not specified, the graph $\G$ specifies which variables are eligible to be input for that function.
The fairness measure $\mu$ is defined by the probability over unobserved $\*U$, and Algorithm \ref{alg.gibbs} is to estimate the exogenous distribution $P(\*U)$ and how it is pushed down to each element through functions $f_V$. These components serve as the foundational input for the algorithm.

\begin{algorithm}[t]
  \caption{\textsc{Sampling Counterfactual Probabilities}: SampleCTF($\*D, \mathcal{G}, \mu $)}
  \label{alg.gibbs}
\begin{algorithmic}[1]
 \REQUIRE $\*D = \left\{(V^t): t = 1, \dots, T \right \}$, $\mathcal{G}$, $\mu$,  $N$; $\alpha$, $M$,$K$
 \STATE Initialize $\mathbf{q}^0 = \frac{1}{K} \cdot (1, \cdots , 1)$ to be a uniform distribution with $K$ cases
 \STATE Initialize $f_V$ as a random function for all $V \in \*V$.
  \FOR {$t = 1, \dots, T$}
  \STATE Randomly assign $U^{t}$
 \STATE Update $f_V$ to match $V^t$ and $U^t$ for all $V \in \*V$
 \ENDFOR
 \COMMENT{ \# Initialization done, sampling starts}
 \FOR {$i = 1, \dots, M+N$}
 \FOR {$t = 1, \dots, T$}
 \STATE Compute $P(U | V = V^t; \mathbf{q}^{i-1}) \propto P(V^t | U) \mathbf{q}^{i-1}$.
 \STATE Draw $U^t \sim P(U | V = V^t; \mathbf{q}^{i-1})$.
 \ENDFOR
 \STATE Initialize Dirichlet Prior $\theta = \alpha \cdot (1, \cdots , 1)$ for $\mathbf{q}$
 \STATE Initialize $f_V$ as a random function for all $V \in \*V$
 \FOR {$t = 1, \dots, T$}
 \STATE Update $f_V$ to match $V^t$ and $U^t$ for all $V \in \*V$
 \STATE Update $\theta$ as the sum of the occurrence of $U$
 \ENDFOR
 \STATE Draw $\mathbf{q}^i \sim \mbox{Dirichlet}\left(\theta\right)$
 \IF {i > M}
 \STATE $\mu_i = \mu(\mathbf{q}^{i}; f_V)$
 \ENDIF
 \ENDFOR
 \RETURN Sampled counterfactual measure $[\mu_{M+1}, \ldots, \mu_{M+N}]$
\end{algorithmic}
\end{algorithm}

In addition to the aforementioned core inputs, several parameters need to be specified. For simplicity and without loss of generality, we assume the existence of only one unobserved variable \(U\) with a finite cardinality of \(K\) in this section. In practical scenarios, there may be multiple unobserved variables \(U_1, U_2, \ldots\), each with corresponding cardinalities \(K_1, K_2, \ldots\) (or one can think of $\*U$ as a vector). Exogenous probabilities $P(u)$ are a probability vector drawn from a Dirichlet distribution, i.e., $P(u) \sim \mbox{Dirichlet}\left (\alpha, \dots, \alpha \right)$, where \(\alpha > 0\) is a small constant chosen for the Dirichlet prior. 

A natural question arising at this point is how to determine the cardinality $K$ for the domain of any exogenous variable $U \in \*U$. To answer this question, we first introduce some necessary concepts in causal inference. We will utilize a special type of clustering of endogenous variables in the causal diagram, which is called \emph{confounded components} \citep{tian:pea02}. For convenience, let a \emph{bi-directed arrow} $V_i \leftrightarrow V_j$ between endogenous nodes $V_i, V_j \in \*V$ be defined as a sequence $V_i \leftarrow U_k \rightarrow V_k$ where $U_k \in \*U$ is an exogenous parent shared by $V_i, V_j$. A \emph{bi-directed path} is a consecutive sequence of bi-directed arrows, ---i.e., a path composed entirely of bi-directed arrows.
Formally,
\begin{definition}[C-Component \citep{tian:pea02}]\label{def:c-component}
	For a causal diagram $\G$, a subset of variables $\*C \subseteq \*V$ is a c-component if any pair of nodes $V_i, V_j \in \*C$ is connected by a bi-directed path.
\end{definition}
For an arbitrary exogenous variable $U \in \*U$, we denote by $\*C(U)$ the c-component covering $U$ in $\G$, i.e., $U \in \bigcup_{V \in \*C(U)} U_V$. For instance, every node in \Cref{fig.standard_a} is a c-component due to the lack of bidirected arrows. On the other hand, exogenous variables $U_1, U_2$ in \Cref{fig.standard_b} are covered by a single c-component $\*C(U_1) = \*C(U_2) = \{A, Z, W, Y\}$ due to the bi-directed path $A \leftrightarrow Z \leftrightarrow Y \leftrightarrow W$.

To compute the cardinality $K$ for unobserved variables $U$, the C-component of $U$ is first identified, and the number of states required to represent all states of the observed variables is counted. More specifically, we will model every unobserved exogenous variable $U_i \in \*U$ as a discrete variable taking values in a finite domain $\{1, \dots, K_i\}$. For every $U_i \in \*U$, the cardinality $K_i$ of the exogenous domain of $U_i$ is bounded by
\begin{align}
    K_i = d_i + 1, \;\;\; \text{ where }d_i= \Pi_{V\in \Pa(\*C(U_i))} |V| .\label{eq:u_bound}
\end{align}
In the above equation, $\*C(U_i)$ is the c-component covering $U_i$ in the causal diagram of the model $\1M$; $\Pa(\*C(U_i))$ are observed direct parents of nodes in the c-component $\*C(U_i)$ (including $\*C(U_i)$). For example, suppose our goal is to infer the counterfactual direct effect $\text{DE}_{a_0, a_1}(y|a)$ from the observational distribution $P(A, Z, W, Y)$ in \Cref{fig.standard_b}; $A, Z, W, Y$ are binary variables taking values in $\{0, 1\}$. Since $U_1, U_2$ share the same c-component $\{A, Z, W, Y\}$, SampleCTF will set their cardinality $K_1, K_2$ are set as $K_1 = K_2 = 2^4 + 1 = 17$.

Since we do not assume any specific forms of the exogenous domains and the exogenous variables $\*U$ could take any values, one may wonder whether bounding the exogenous cardinality following Eq.~\ref{eq:u_bound} could impose additional restrictions, rendering the inferred posterior counterfactual probabilities invalid. We will next show this is not the case. First, Proposition 2.6 in \cite{zhang2022partial} for any structural causal model $\1M$ with discrete observed variables $\*V$, one could generate its observational distribution $P(\*V)$ using discrete latent variables $\*U$, where for every $U_i \in \*U$, its cardinality is bounded by $d_i$ defined in Eq.~\ref{eq:u_bound}. We could extend this discretization result to represent the observational data $P(\*V)$ and the target counterfactual measure $\mu$ simultaneously. The additional $+1$ state in Eq.~\ref{eq:u_bound} allows us to represent the probability mass associated with the target counterfactual measurement $\mu$. The following proposition ensures that this exogenous cardinality bound is sufficient in representing the observed data $P(\*V)$ and the target fairness measure $\mu$, without violating qualitative knowledge encoded in graph $\G$.
\begin{theorem}\label{thm:bound}
    For an SCM $\1M$, let $\1G$ be its associated causal graph, $P(\*V)$ be its observational distribution, and $\mu$ be a nested counterfactual measure. Then there exists an alternative SCM $\1N$ with exogenous cardinalities bounded in Eq.~\ref{eq:u_bound} such that $\1M$ and $\1N$ induce the same $\1G$, $P(\*V)$ and $\mu$.
\end{theorem}
\begin{proof}
        Proposition 2.6 in \cite{zhang2022partial} for any structural causal model $M$ with discrete observed variables $\*V$, one could generate its observational distribution $P(\*V)$ using discrete latent variables $\*U$, where for every $U_i \in \*U$, its cardinality is bounded by $d_i$ defined in Eq.~\ref{eq:u_bound}. The additional $+1$ state in Eq.~\ref{eq:u_bound} allows one to represent the probability mass associated with the target counterfactual measure $\mu$. The proof follows a similar procedure for \cite[Lemma A.6]{zhang2022partial}.  
\end{proof}
After initializing exogenous cardinalities, \Cref{alg.gibbs} begins by initializing the distribution of unobserved variables \(P_0(U)\) and the functions \(f_V\) (see \Cref{alg.gibbs}, Steps 1 to 6). It chooses a uniform distribution over $U$, i.e., \(P_0(U)=\mathbf{q}^0 = \frac{1}{K} \cdot (1, \cdots , 1)\) (Step 1). The initialization of \(f_V\) involves randomly defining the function (Step 2) and drawing \(U^t\) based on the uniform distribution \(P_0(U)\) (Step 4). The structural function \(f_V\) is then updated to match \(V^t\) and \(U^t\) by iterating over the data one by one. The order of the samples can be randomized if desired, as the goal is to have random functions that align well with the data.

In each round \(i\) of the sampling procedure (\Cref{alg.gibbs}, Steps 8 to 23), new \(U^t\) is drawn from \(P(U | V = V^t; \mathbf{q}^{i-1})\), and \(f_V\) is re-initialized accordingly. While we update the $f_V$, a Dirichlet prior $\theta$ is initialized to $\theta = \alpha \cdot (1, \cdots , 1)$ and updated to obtain a Dirichlet posterior based on the occurrence of \(U\). Finally, an updated distribution $\mathbf{q}^{i}$ over $U$ is drawn from the Dirichlet posterior: \(\mathbf{q}^{i} \sim \mbox{Dirichlet}\left(\theta\right)\). This process is iterated for $M$ rounds to converge to a stable period, and the last $N$ rounds are used for generating counterfactual probabilities in Step 20, computed using $\mathbf{q}^{i}$ and $f_V$. In summary, \Cref{alg.gibbs} samples counterfactual probabilities by exploring the entire function space while updating the conditional probability of unobserved variables given observations.
\begin{theorem}\label{thm:sample}
    Given priors $\*\rho$ over the exogenous probabilities $P(\*U)$ and structural functions $\1F$, SampleCTF$(\*D, \1G, \*\mu)$ draws a posterior sample of the target counterfactual measure $\*\mu$ conditioning on the observed data $\*D$, i.e., $\mu \sim P(\mu \mid \*D; \1G, \*\rho)$.
\end{theorem}
\begin{proof}
    It follows from Theorem \ref{thm:bound} that discrete exogenous domains with cardinality bounds in Eq.~\ref{eq:u_bound} are sufficient in representing the observational distribution $P(\*V)$ and an arbitrary nested counterfactual measure $\mu$ in a casual graphical model $\G$. Given priors $\*\rho$ of the exogenous probabilities $P(\*u)$ and structural functions $\1F$, the statement follows from the validity of Gibbs sampling \citep{geman1984stochastic}.
\end{proof}
Choosing the appropriate burn-in period $M$ involves recognizing that Algorithm \ref{alg.gibbs} operates as a special type of Gibbs sampling method. Consequently, we anticipate the sampled $P(U)$ to converge to a certain distribution. As the true $P(U)$ is inaccessible, we suggest computing a probability obtainable from the observed distribution and verifying that the sampled quantity indeed converges to the observed quantity. For instance, in \Cref{sec:compas} that follows, we calculate $P\Parens{Y=1 \mid A=0}$ from the sampled $P(U)$ and the functions, comparing it against the same quantity computed from the observed data (assumed to be the ground truth).

%% file: section4.tex
\section{Simulation study}
\label{sec:simulation}
In this section, we evaluate the accuracy of our method using a simulation dataset. The data was generated based on the graphical model in \Cref{fig.standard_a}, with a random distribution for each unobserved variable. The final dataset consists of 50,000 observations for the binary variables $Z$, $A$, $W$, and $Y$.

By applying the FCI algorithm \citep{spirtes2001anytime}, one could infer that the covariate $Z$ and the mediator $W$ are not adjacent. That is, there is no direct $Z \rightarrow W$, $Z \leftarrow W$, or bidirected arrow $Z \leftrightarrow W$ in the underlying causal diagram. We also assume access to the domain knowledge that variables $A, Z, W$ could be potential direct causes for the outcome $Y$, but not vice versa. Synthesizing the domain knowledge with the causal relationships inferred from the observational data leads to a causal diagram $\1G$ described in \Cref{fig.standard_b}. Compared with the ground-truth graph in \Cref{fig.standard_a}, our learned model $\1G$ has additional latent variables $U_1, U_2$ affecting variable clusters $\{A, Z, Y\}$ and $\{A, W, Y\}$ respectively, since the combination of the observational data and domain knowledge is unable to rule out the presence of unobserved confounding.

With this $\mathcal{G}$, the cardinality of unobserved variables $U_1$ and $U_2$ should be greater than or equal to $17$. We set $K=22$ and apply \Cref{alg.fairness} to obtain bounds for the direct effect (DE), indirect effect (IE), and total effect (TE) of changing $A$.
\Cref{fig.simu.hist} shows the histograms of the DE, IE, and SE drawn by \Cref{alg.gibbs}.
The results consistently contained the ground truth within the $(1 - \delta)\%$ confidence interval for all three quantities with $\delta = 0.05$, corroborating Theorems \ref{thm:bound} and \ref{thm:sample}. 
More detailed results are shown in Appendix \ref{app.simu}.

\begin{figure*}[t]
    \centering
    \includegraphics[width=0.33\linewidth]{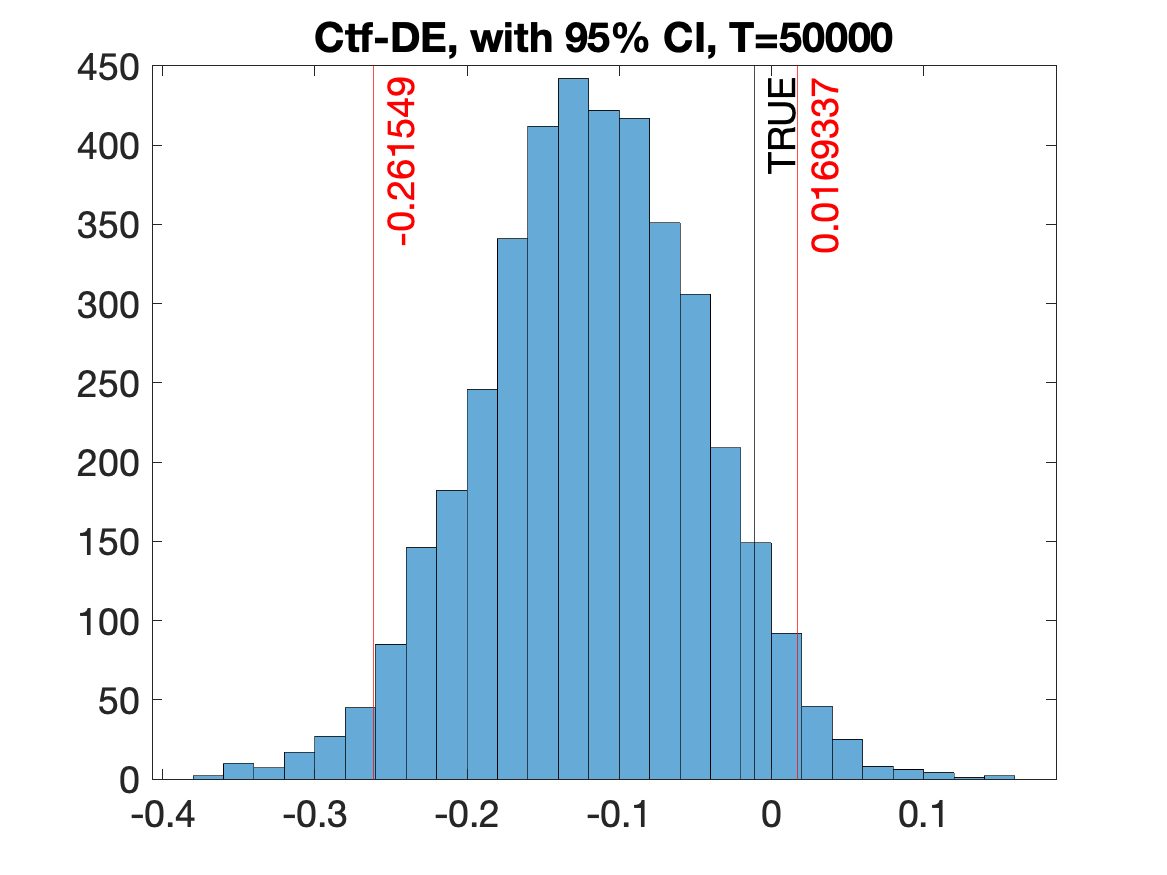}
    \includegraphics[width=0.33\linewidth]{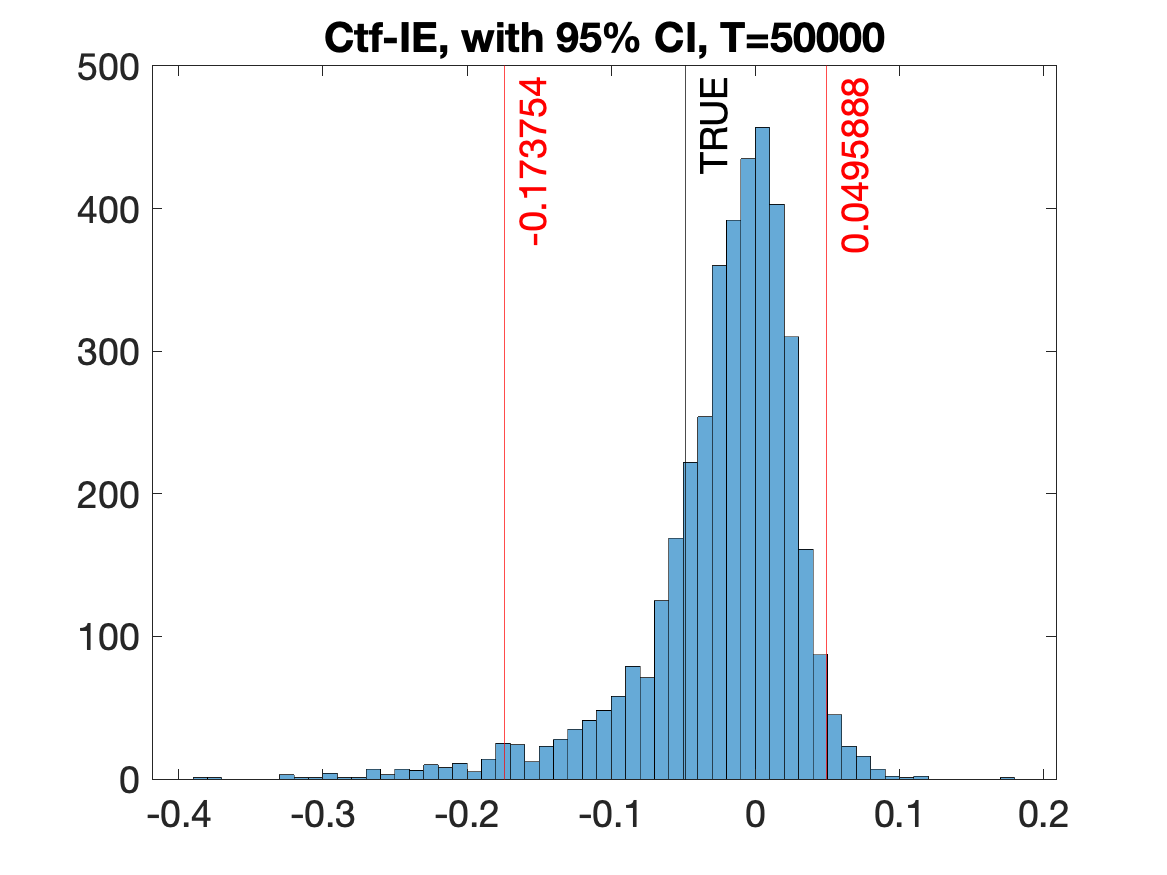}
    \includegraphics[width=0.33\linewidth]{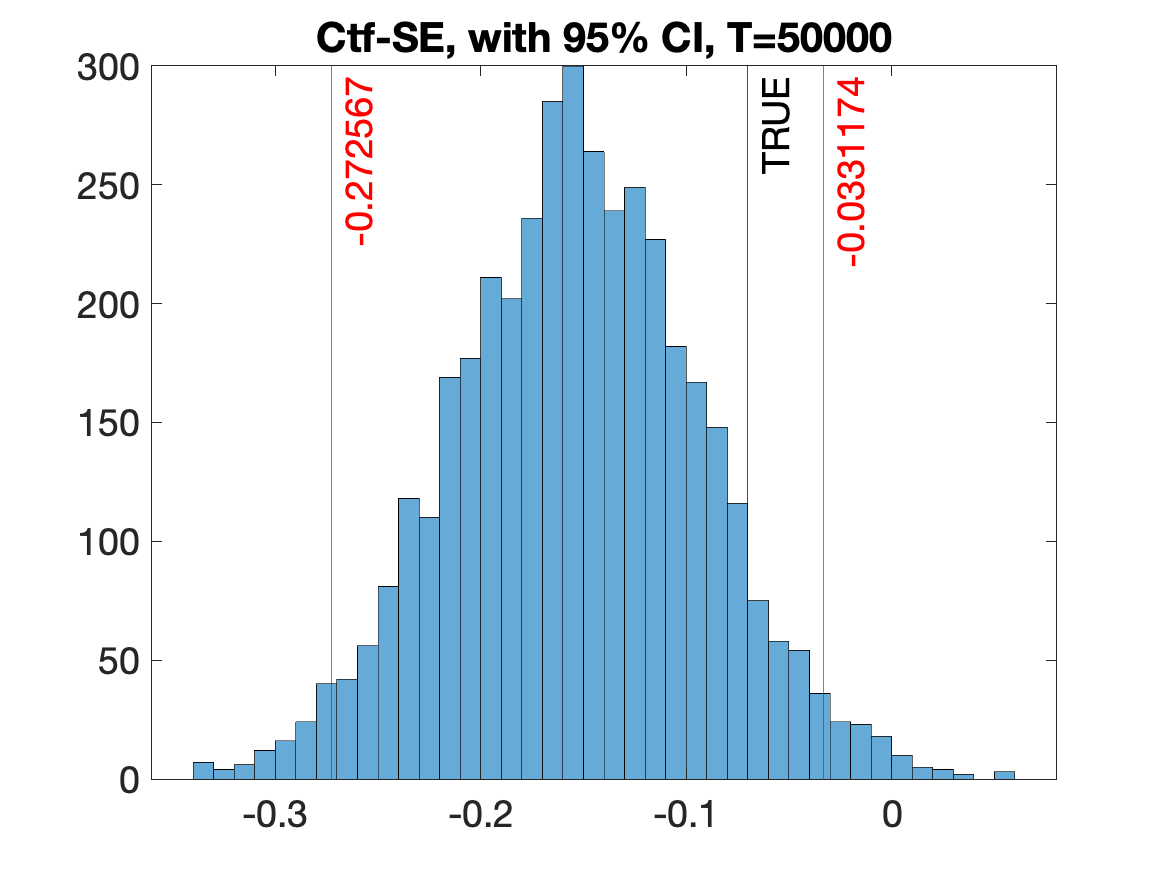}
    \caption{Histograms for DE, IE, and SE, obtained from the simulation dataset. The black vertical line is the ground-truth value (labeled as TRUE) and the two red lines show 95\% confidence interval (2.5\% top, 2.5\% bottom).}
    \label{fig.simu.hist}
\end{figure*}

%% file: section5.tex
\section{COMPAS Case study}
\label{sec:compas}

In this section, we demonstrate \Cref{alg.fairness}, IDFair, on a real-world dataset. We describe the COMPAS dataset in Section \ref{sec:data}, identify its causal graphical structure in Section \ref{sec:scm}, and share the results of the sampling DE, IE, and SE in Section \ref{sec:results}.
Then, we instantiate the causal explanation formula from \cite{zhang2018fairness} and show how our results aligns with it in Section \ref{sec:tv}.
Finally, we use our algorithm to estimate other counterfactual fairness measures and compare them against DE, IE, and SE in Section \ref{sec:comparison}.

\subsection{COMPAS Dataset}\label{sec:data}

COMPAS (Correctional Offender Management Profiling for Alternative Sanctions) is an algorithm developed by Northpointe (now Equivant) used in the judicial process to predict a criminal defendant's recidivism score.
In 2016, ProPublica raised a concern that black defendants were often predicted to be at a higher risk of recidivism than they actually were (45\% vs. 23\%), whereas white defendants were often predicted to be less risky than they were (48\% vs. 28\%) based on a 2-year follow-up data \citep{propublica}.
ProPublica publicized their dataset on Github\footnote{https://github.com/propublica/compas-analysis/}, which contains COMPAS scores for \emph{Risk of Recidivism} ranging from $1$ to $10$, as well as each defendant's race, sex, age, criminal history (prior counts), and charge degree.

In our analysis, we assume $A = \{ \mbox{Race, Sex, Age} \}$ be protected attributes that shall not affect the recidivism score directly, $X = \{ \mbox{Charge degree, Prior counts} \}$ be measurements that could be related to an individual's actual recidivism ($X$ could be either confounder $Z$ or mediator $W$), and $Y = \{ \mbox{Risk of Recidivism} \}$ be the outcome of the algorithm to be assessed if it is fair or not.
Table \ref{t.variables} shows the definition of our variables.

\begin{table}[ht]
\caption{Definition of variables}
\label{t.variables}
\begin{center}
\begin{tabular}{@{}ccc@{}}
\toprule
Variable           & 0             & 1                \\ \toprule
Race($A$)          & Others        & African-American \\
Age($A$)           & Less than 30  & Over 30          \\
Sex($A$)           & Female        & Male             \\ \midrule
Charge Degree ($W_1$) & Misdemeanor   & Felony           \\
Prior Counts ($W_2$)  & $\leq 2$ & $>2$  \\ \midrule
Score ($Y$, $1 \sim 10$)        & $\leq 5$ & $>5$  \\ \bottomrule
\end{tabular}
\end{center}
\end{table}

\subsection{Causal Graphical Structure of COMPAS}\label{sec:scm}

This section illustrates Steps 1-2 of \Cref{alg.fairness}, where the graphical structure $\mathcal{G}$ is learned using the FCI algorithm, implemented in an R package \texttt{pcalg} \citep{kalisch2012causal}. First, we come up with an equivalent class of candidate causal diagrams using FCI algorithm, then use qualitative domain knowledge to fine-tune the result and filter out incompatible diagrams. This approach allows non-Markovian cases where unobserved confounders generally exist. We assume that there exist no direct edges from $X$ (Charge degree, Prior counts) to the protected attributes $A$ (Age, Sex, or Race).

We show in \Cref{fig.equiv.class} in Appendix \ref{app.equiv.class}, the equivalent class inferred by the FCI algorithm. Based on this, we leverage our background knowledge and make informed decisions about edge types.
First, we assume that there cannot be a direct edge between the protected attributes, i.e., age cannot be influenced by sex.
Next, considering that the determination of protected attributes precedes the establishment of charge degree or prior counts, we preclude any directed edges from charge degree or prior counts to the protected attributes. 
Similarly, since prior counts are already decided before the charge degree, we prohibit directed edges from charge degree to prior counts. Finally, we permit bi-directed edges where such a possibility exists.
The finalized graphical model assumed throughout the analysis is shown in \Cref{fig.scm}.

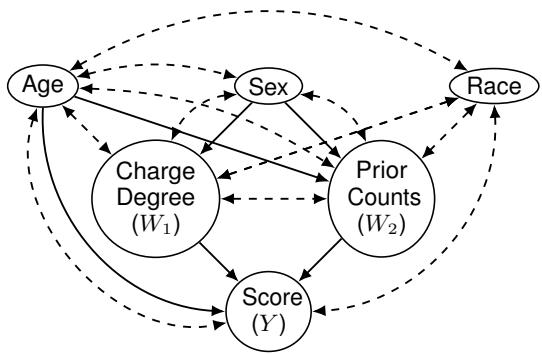
\begin{figure}[ht]
    \centering
    \begin{tikzpicture}
			\def\outerr{3.2}
			\def\innerr{3}

			\node[textnode] (A) at (0, 0) {Age};
			\node[textnode] (S) at (3, 0) {Sex};
			\node[textnode] (R) at (6, 0) {Race};
			\node[textnode, text width=1.1cm, align=center] (W1) at (1.5, -1.5) {Charge Degree ($W_1$)};
			\node[textnode, text width=0.9cm, align=center] (W2) at (4.5, -1.5) {Prior Counts ($W_2$)};
			\node[textnode, text width=0.7cm, align=center] (Y) at (3, -3) {Score ($Y$)};

			\draw[dir] (A) -- (W2);
			\draw[dir] (S) -- (W2);
                \draw[dir] (S) -- (W1);
                \draw[dir] (A) to [bend right = 45] (Y);
                \draw[dir] (W1) -- (Y);
                \draw[dir] (W2) -- (Y);
			\draw[bidir] (A) to [bend left = 15] (S);
			\draw[bidir] (A) to [bend left = 30] (R);
			\draw[bidir] (A) to [bend right = 60] (Y);
			\draw[bidir] (R) to [bend left = 45] (Y);
			\draw[bidir] (R) -- (W2);
                \draw[bidir] (R) -- (W1);
                \draw[bidir] (R) -- (W1);
                \draw[bidir] (W1) -- (W2);
                \draw[bidir] (A) -- (W1);
                \draw[bidir] (S) to [bend right = 30] (W1);
                \draw[bidir] (S) to [bend left = 30] (W2);
                \draw[bidir] (A) to [bend left = 15] (W2);
    \end{tikzpicture}
    \caption{Inferred causal diagram for the COMPAS dataset using the FCI algorithm and domain knowledge.}
    \label{fig.scm}
\end{figure}

From this composite graphical model, we can obtain a simplified version for each protected attribute. \Cref{fig.case} shows graphical models for (a) $A = \{\text{Race}\}$, (b) $A = \{\text{Age}\}$, and (c) $A = \{\text{Sex}\}$. In this case, we are only considering one graph for each case, i.e., the equivalence class $\1E^*$ in Step 2 of \Cref{alg.fairness} has only one candidate.

Finally, we obtain the minimum number of states $K$ for each exogenous variable $U$ as described in Section \ref{sec:hyperparam}
Given that all endogenous variables are binary, it follows that for exogenous variables $U_1$ and $U_2$, we require a minimum of $K_1=K_2=2^4+1=17$, $K_1=K_2=2^5+1=33$, or $K_1=K_2=2^6+1=65$ states when $A=\text{Race}$, $A=\text{Age}$, or $A=\text{Sex}$, respectively.
For the experiments, both $K_1$ and $K_2$ were set to have the same value of $20, 40,$ and $70$ for $A=\text{Race}$, $A=\text{Age}$, and $A=\text{Sex}$, respectively.

\begin{figure}[t]
    \centering
    \hfill
    \begin{subfigure}{0.33\linewidth}\centering
            \begin{tikzpicture}
			\def\outerr{3.2}
			\def\innerr{3}

			\node[vertex] (A) at (0.5, -1.5) {A};
			\node[vertex] (Y) at (0.5, -3) {Y};
			\node[vertex] (X1) at (-0.5, -1.5) {$W_1$};
			\node[vertex] (X2) at (1.5, -1.5) {$W_2$};
			\node[uvertex] (U1) at (0, 0) {$U_1$};
			\node[uvertex] (U2) at (1, 0) {$U_2$};

			\draw[dir] (X1) -- (Y);
			\draw[dir] (X2) to [bend left = 15] (Y);
			\draw[unobsdir] (U1) -- (A);
			\draw[unobsdir] (U1) -- (X1);
			\draw[unobsdir] (U1) -- (X2);
			\draw[unobsdir] (U2) -- (A);
			\draw[unobsdir] (U2) to [bend left = 15]  (Y);
		\end{tikzpicture}
            \caption{Race}
            \label{fig.case.a}
        \end{subfigure}\hfill
	\begin{subfigure}{0.33\linewidth}\centering
        \begin{tikzpicture}
			\def\outerr{3.2}
			\def\innerr{3}

			\node[vertex] (A) at (0.5, -1.5) {A};
			\node[vertex] (Y) at (0.5, -3) {Y};
			\node[vertex] (X1) at (-0.5, -1.5) {$W_1$};
			\node[vertex] (X2) at (1.5, -1.5) {$W_2$};
			\node[uvertex] (U1) at (0, 0) {$U_1$};
			\node[uvertex] (U2) at (1, 0) {$U_2$};

			\draw[dir] (X1) -- (Y);
			\draw[dir] (X2) to [bend left = 15] (Y);
                \draw[dir] (A) -- (X2);
                \draw[dir] (A) -- (Y);
			\draw[unobsdir] (U1) -- (A);
			\draw[unobsdir] (U1) -- (X1);
			\draw[unobsdir] (U1) -- (X2);
			\draw[unobsdir] (U2) -- (A);
			\draw[unobsdir] (U2) to [bend left = 15] (Y);
		\end{tikzpicture}
            \caption{Age}
            \label{fig.case.b}
    \end{subfigure}\hfill
    \begin{subfigure}{0.33\linewidth}\centering
    \begin{tikzpicture}
			\def\outerr{3.2}
			\def\innerr{3}

			\node[vertex] (A) at (0.5, -1.5) {A};
			\node[vertex] (Y) at (0.5, -3) {Y};
			\node[vertex] (X1) at (-0.5, -1.5) {$W_1$};
			\node[vertex] (X2) at (1.5, -1.5) {$W_2$};
			\node[uvertex] (U1) at (0, 0) {$U_1$};
			\node[uvertex] (U2) at (1, 0) {$U_2$};

			\draw[dir] (X1) -- (Y);
			\draw[dir] (X2) to [bend left = 15] (Y);
                \draw[dir] (A) -- (X1);
                \draw[dir] (A) -- (X2);
			\draw[unobsdir] (U1) -- (A);
			\draw[unobsdir] (U1) -- (X1);
			\draw[unobsdir] (U1) -- (X2);
			\draw[unobsdir] (U2) -- (A);
			\draw[unobsdir] (U2) to [bend left = 15] (Y);
		\end{tikzpicture}
            \caption{Sex}
            \label{fig.case.c}
    \end{subfigure}\null
    \caption{Graphical model for each protected variable with two exogenous variables. Each exogenous variables have 17 states. $A$ denotes (a) race, (b) age, (c) sex, $W_1$ denotes charge degree, $W_2$ denotes prior counts, and $Y$ denotes the predicted COMPAS score.}
    \label{fig.case}
\end{figure}
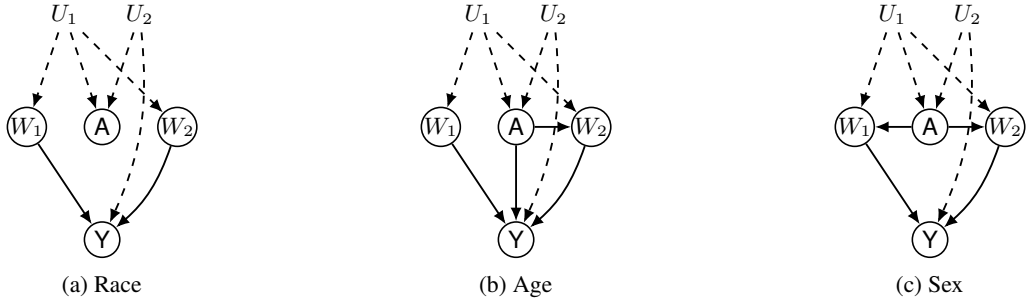

\subsection{Estimated Counterfactual DE, IE,and SE}\label{sec:results}
In this section, we show the results of running \Cref{alg.gibbs} on COMPAS dataset to obtain bounds for counterfactual fairness measures. We start with a uniform initialization for $P(U_1)$ and $P(U_2)$, using respective $K_1$ and $K_2$. The burn-in period $M$ was set to $2000$, and $N=4000$ samples were collected after the burn-in for all experiments. Throughout all experiments, we consistently set the error rate $\delta = 0.05$. However, as previously mentioned, one could specify $\delta$ to be any reasonable real value in $[0, 1)$ based on the analysis. We highlight main findings here; more visualizations and discussion are deferred to Appendix \ref{app.compas}.

\subsubsection{(a) A=Race}\label{sec:race}

When we analyze the counterfactual effects of Race, the direct and indirect effects were zero by the graphical structure---there is no direct or indirect path from $A$ to $Y$. The first histogram in \Cref{fig.compas.hist} shows the distribution of $4000$ samples' $SE$ we obtained from Gibbs sampling. The width of the $95\%$ confidence interval is $0.0423$ (about $4\%p$). In words, this means that if a non-African-American individual had been African-American, the probability that the individual is assigned to a high score (greater than $5$) would have become at least about $23\%$ higher and at most about $27\%$ higher, solely through the spurious paths (backdoor paths).

\begin{figure*}[t]
    \centering
    \includegraphics[width=0.33\linewidth]{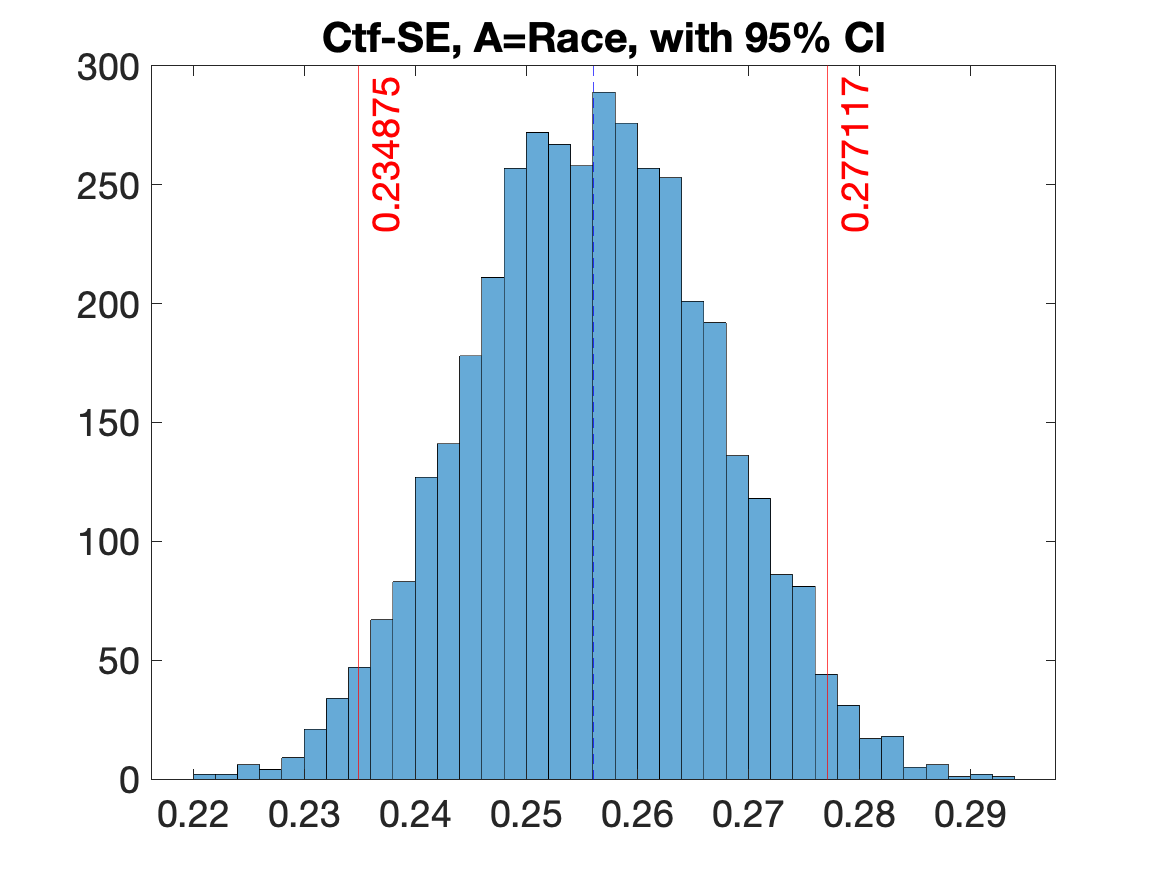}
    \includegraphics[width=0.33\linewidth]{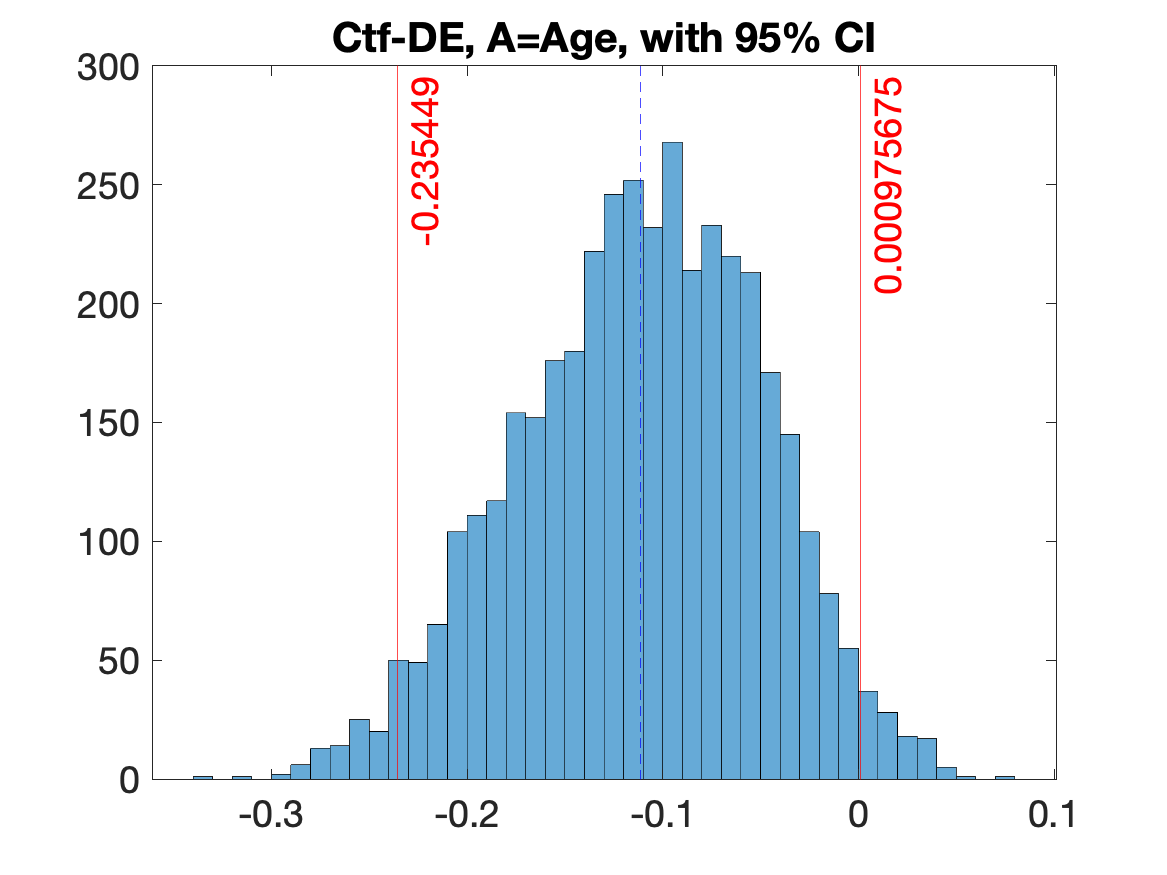}
    \includegraphics[width=0.33\linewidth]{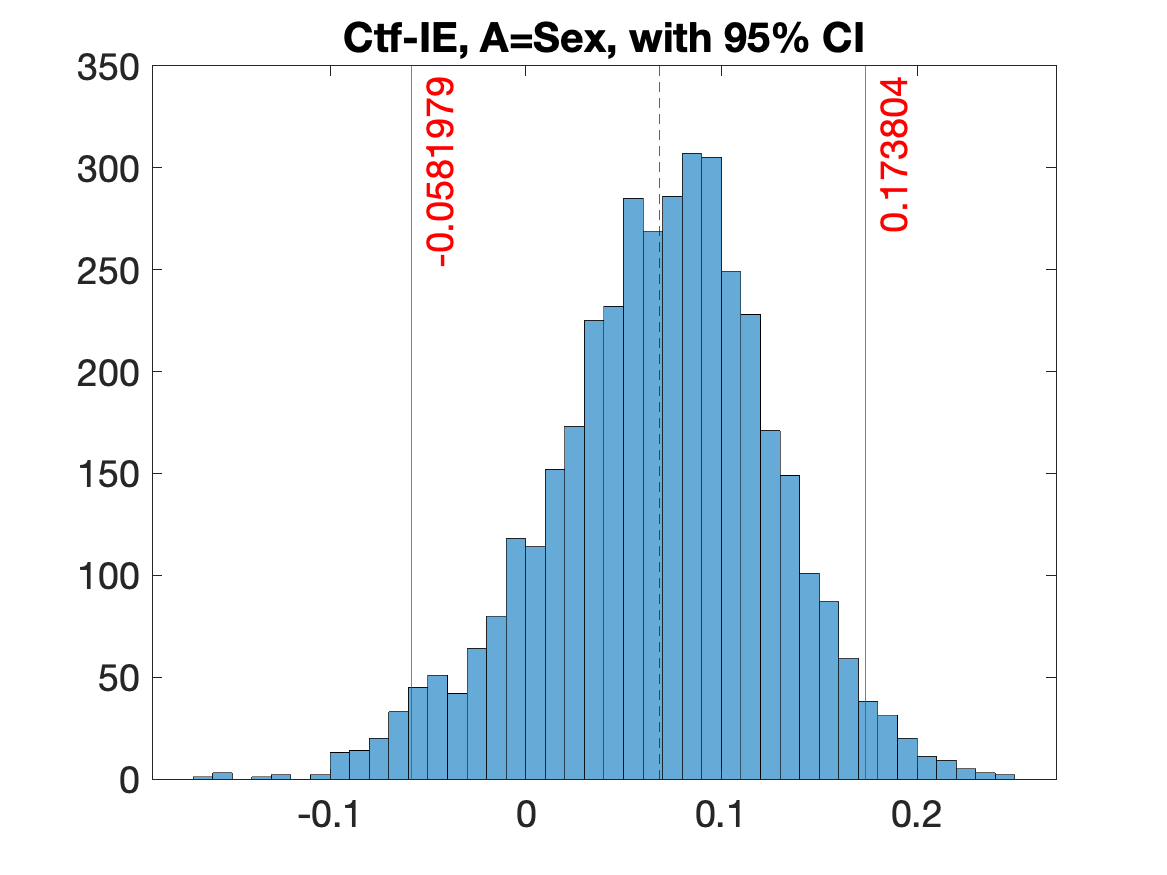}
    \caption{Histogram of SE when $A=\mbox{Race}$ (left), DE when $A=\mbox{Age}$ (middle), IE when $A=\mbox{Sex}$ (right). The two red lines show 95\% confidence interval (2.5\% top, 2.5\% bottom).}
    \label{fig.compas.hist}
\end{figure*}

\subsubsection{(b) A=Age}\label{sec:age}

The second histogram in \Cref{fig.compas.hist} shows the confidence interval for DE, $(-0.2354, 0.0009)$, when $A=\text{Age}$. The CI is wide but still indicates some portion of negative effect, i.e., if you were older, you are less likely to be assigned high risk. The confidence interval for IE is centered around zero and ranges from $-10$\% to $10$\%, hence it is hard to draw any conclusion about the indirect effect.
The confidence interval of SE is $41.68\%p$ and this bound is too wide to draw any causal conclusion.

\begin{table*}[t]
\begin{center}
\caption{95\% Confidence interval for counterfactual family of effects and Total variation of each protected attribute ($A$).}
\label{table.tv}
\begin{tabular}{@{}cccccr@{}}
\toprule
$A$  & $DE_{A=1,A=0}|A=1$ & $IE_{A=0,A=1}|A=1$ & $SE_{A=0,A=1}$ & TV Bound & TV \\ \midrule
Race 
& (0, 0)
& (0, 0)
& (0.2348, 0.2771)
& (0.2348, 0.2771)
& $0.2544$ \\
Age  
& (0.0333, 0.4511)
&(-0.1113, -0.0352)
&(-0.2609, 0.3164)
&(-0.3399, -0.1413)
& $-0.2018$\\
Sex  
& (0, 0)
&(-0.0196, 0.1656)
&(-0.1284, 0.1521)
&(0.0066, 0.0872)
& $0.0629$\\ \bottomrule
\end{tabular}
\end{center}
\end{table*}

\subsubsection{(c) Sex}\label{sec:sex}

When $A=\text{Sex}$, similar to the Race case, the DE is zero by definition (no direct path from $A$ to $Y$). 
The last histogram in \Cref{fig.compas.hist} shows a wide confidence interval for IE ($23.1\%p$, CI=$(-0.0581,0.1738)$). This could stem from the inherent complexity of the graphical model, as evidenced by the wider spectrum shown in the convergence graph (\Cref{fig.sex.converge} in Appendix \ref{app.compas}).

\subsubsection{Comparison to Total Variation (TV)}\label{sec:tv}

Lastly, we demonstrate Theorem 1 of \cite{zhang2018fairness} with the results we obtained, which shows total variation as a linear combination of DE, IE, and SE.
We take the first version of the theorem with $a_0=0$ and $a_1=1$:
\begin{align*}
    TV_{A=0,A=1}(Y=1) = &SE_{A=0,A=1}(Y=1) + IE_{A=0,A=1}(Y=1|A=1) \\
&- DE_{A=1,A=0}(Y=1|A=1).
\end{align*}
\Cref{table.tv} shows the values for each term in the above equation. The second last column (TV Bound) indicates the bound for the $TV_{A=0,A=1}(Y=1)$, and the total variation (the last column) falls within the TV bound for all three cases.

\subsection{Comparison to Other Fairness Measures}\label{sec:comparison}

This section compares the quantitative and qualitative findings using various fairness measures with our results.
First, we use our algorithms to estimate other \emph{counterfactual} measures, such as the Counterfactual Effect (CE) based on the counterfactual fairness definition by \cite{kusner2017counterfactual} and the Path-specific Counterfactual Effect (PSE) by \cite{chiappa2019path}, and compare with our findings.
Then, we qualitatively compare our results to previous findings with different statistical fairness measures.

\subsubsection{Counterfactual Effect (CE)}
When $A=\text{Race}$, as in \Cref{fig.case.a}, the $CE$ is zero by the graphical structure---there is no direct or indirect path from $A$ to $Y$. This means that the COMPAS algorithm abides by the counterfactual fairness by \cite{kusner2017counterfactual} with respect to race.

With $A=\text{Age}$, we calculate the effect of changing $A=0$ (baseline, age less than $30$) to $A=1$ (intervention, age over $30$) conditioned on $A=0$, $W_1$, and $W_2$.
Since the sampling did not converge when $W_2=0$, we present CE conditioned on $W_2=1$ only.
The $95\%$ confidence interval for CE conditioned on $A=0, W_1$ and $W_2=1$ were estimated to be $(0.0290,0.3665)$ when $W_1=0$ and $(0.0120, 0.2544)$ when $W_1=1$ (\Cref{fig.comparison.hist}, left and middle).
The confidence interval lies on the positive side, regardless of the conditioning value for $W_1$. This means that changing age from less than $30$ to over $30$ has a positive counterfactual effect on the COMPAS score, although the precise magnitude is challenging to ascertain due to the wide confidence interval.
Although these results do not necessarily contradict the results from Section \ref{sec:age}, this definition fails to capture the negative DE we obtained.

With $A=\text{Sex}$, we compute the effect of changing $A=0$ (baseline, female) to $A=1$ (intervention, male) conditioned on $A=0, W_1$, and $W_2$.
The counterfactual effect (CE) bound exhibits a relatively positive trend when $W_2 = 0$ (fewer prior charges; bounds are $(0.0794,0.4737)$ when $W_1=0$, $(-0.0626,0.3098)$ when $W_1=1$) and a negative trend when $W_2 = 1$ (more prior charges; bounds are $(-0.3556,0.0658)$ when $W_1 = 0$, $(-0.4081,-0.0200)$ when $W_1=1$)---see Appendix \ref{app.othermeasure}, \Cref{fig.sex.ctf}. This indicates that a female is more likely to receive a higher COMPAS score if she were male, particularly when she has fewer prior charges. However, this is reversed when she has more than two prior charges; she is likely to receive a lower COMPAS score had she been male if $W_2=1$.

\subsubsection{Path-Specific Effect (PSE) with respect to Age}

For $A=\text{Age}$, one can define the direct edge from $A$ to $Y$ to be unfair and the indirect path from $A$ to $W_2$ to $Y$ to be fair. In this case, The PSE is identical to the unconditioned $DE$, and the obtained bound $(-0.0043,0.0983)$ mostly lies on the positive side (\Cref{fig.comparison.hist}, right).

Another way to define an \emph{unfair} path is to choose the path from $A$ to $W_2$.
We can estimate this PSE by setting $A=0$ as the input for $f_Y$ and $A=1$ as the input for $f_{W_2}$.
The obtained $95\%$ confidence interval $(0.0513,0.2426)$ indicates that the direction of the PSE is positive when we change from $A=0$ to $A=1$ ((Appendix \ref{app.othermeasure}, right side of \Cref{fig.age.pse}).
This aligns with Kusner's $CE$ estimation, although the magnitude of the effect is different.

\begin{figure*}[t]
    \centering
    \includegraphics[width=0.33\linewidth]{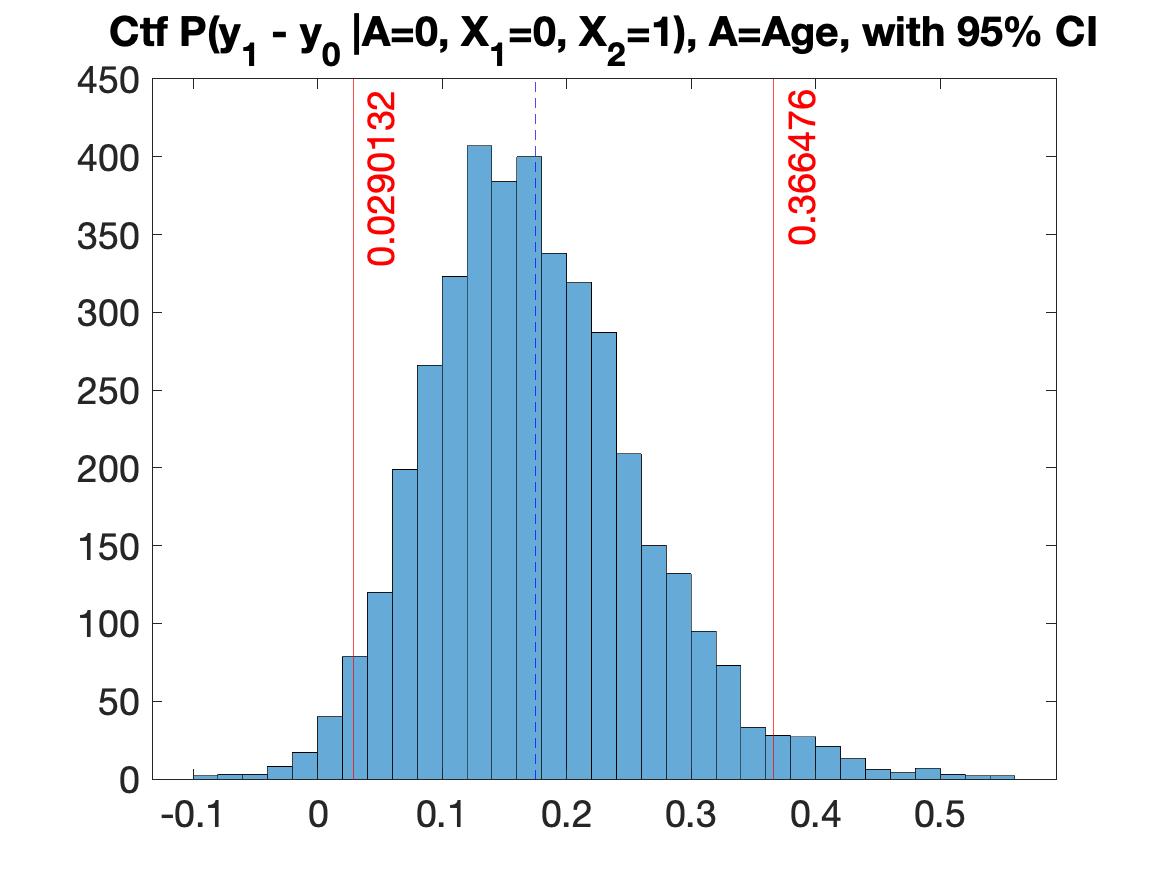}
    \includegraphics[width=0.33\linewidth]{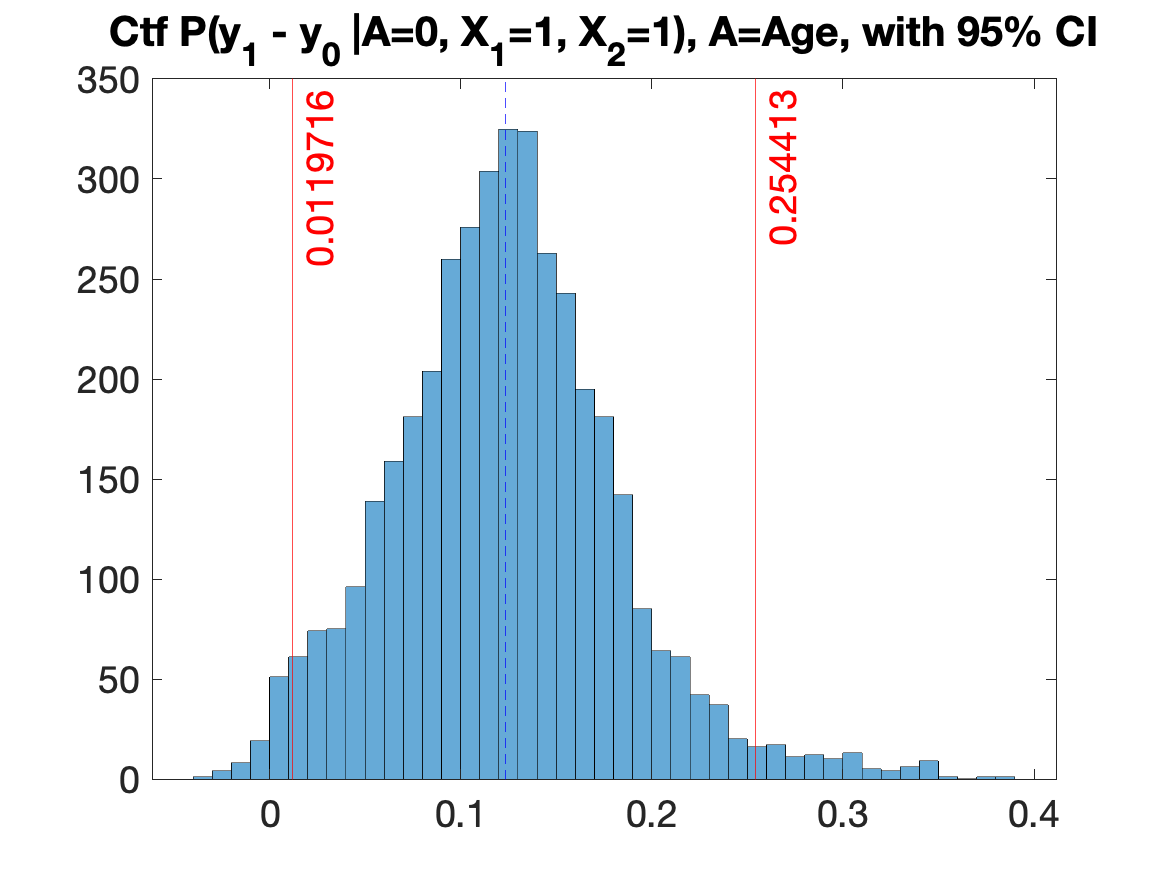}
    \includegraphics[width=0.33\linewidth]{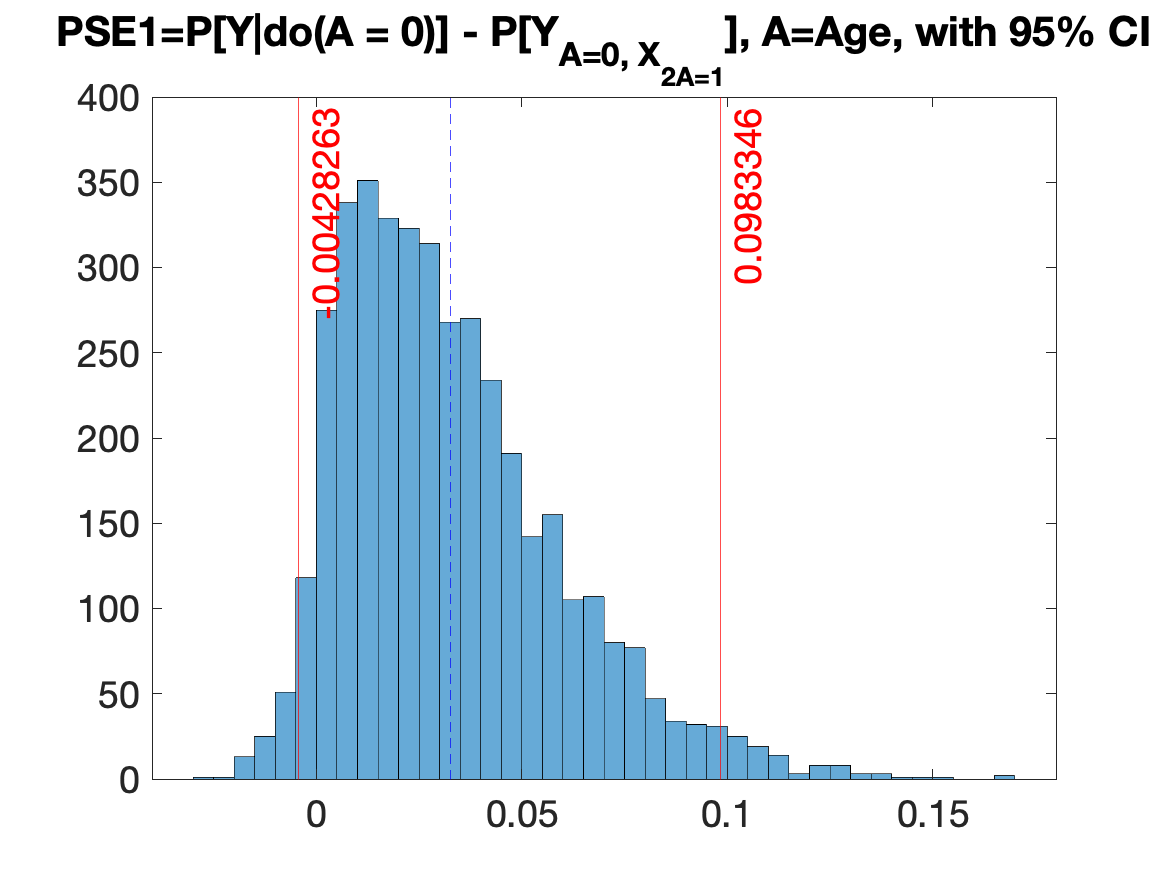}
    \caption{Histograms of other counterfactual fairness measures for $A=\mbox{Age}$. Each histogram shows CE conditioned on $A=0, W_1=0, W_2=1$ (left), CE conditioned on $A=0, W_1=1, W_2=1$ (middle), PSE for the path $A \to Y$ (right). The two red lines show 95\% confidence interval (2.5\% top, 2.5\% bottom). More graphs from Section \ref{sec:comparison} are deferred to Appendix \ref{app.othermeasure}}
    \label{fig.comparison.hist}
\end{figure*}

\subsubsection{Other (Non-counterfactual) Fairness Measures}

The first accusation against COMPAS scores was that the false positive rate was higher for black defendants than for white defendants, and the false negative rate was higher for white defendants than for black defendants \citep{propublica}.
While we cannot directly compare our findings to analyses of false positive or false negative rates---since the notion of counterfactual fairness does not necessarily account for the actual outcome---we do find evidence of racial bias in COMPAS scores through a spurious path. This spurious effect may explain the differences in false positive and false negative rates between races, favoring white defendants.

It was also suggested that the COMPAS score assigns higher score to women than men \emph{``[b]ecause women reoffend at lower rates than men with similar criminal histories''} \citep{corbett2023measure}. Based on our analyses, the direct path from $A=\mbox{Sex}$ to the COMPAS score ($Y$) was absent. If the aforementioned claim were to be true, having a direct path connecting $A$ (Sex) and $Y$ (COMPAS Score) could help mitigate the currently observed unfair behavior against women.

%% file: section6.tex
\section{Discussion}\label{sec:discussion}

\paragraph{\bf Interpretation of Estimated Bounds}
The confidence interval output by Algorithm \ref{alg.fairness} is obtained by testing any possible distribution $P(U)$ over hidden variables and the function $f_V$ determining observed variables against the observation data. For example, $95\%$ confidence interval means that of all $P(U)$ and $f_V$ that can generate the observation data, $95\%$ of the cases will result in the (counterfactual) fairness measure that falls within the calculated confidence interval. Hence, the confidence level accounts for the randomness from the ignorance region (non-identifiable parts in the structure), as well as having a finite number of samples. If desired, our approach can be utilized to approximate the worst-case bound by taking the minimum and the maximum of the samples probabilities. In such cases, increasing the number of samples $N$ may also help.

\paragraph{\bf Choice of Fairness Definition}

In the realm of counterfactual fairness, a variety of proposed definitions often lack alignment with one another, leading to instances where certain definitions are deemed inadmissible with respect to others. A notable example is the CE, which is inadmissible with respect to Direct Effect (DE), Indirect Effect (IE), and Spurious Effect (SE) \citep{plecko2022causal}. This implies that observing $CE=0$ does not guarantee the non-zero DE, IE, or SE. This phenomenon was evident in the case of $A=\text{Race}$, where CE (or PSE) was inherently zero, but a significant SE was estimated. Thus, the choice of fairness metric becomes crucial, as it can lead to divergent conclusions.

Our analysis of fairness concerning Race serves as an illustrative example. Among all the tested definitions, only SE was capable of capturing racial bias, exhibiting a relatively narrow confidence interval around $25\%$. Some may assert that the spurious effect is not \emph{causal}, operating solely through a backdoor path. However, we propose that it should still be regarded as an unfair path, particularly when the testing variable $Y$ represents an algorithmic output, and the variables in the spurious paths are not exhaustively known.

Consider the spurious path in the Race graph involving, for instance, $U_2=\mbox{ the zip code of residence}$. If the outcome variable of interest $Y$ were the actual observed recidivism in the real world, we might contend that the spurious path is not unfair due to its lack of causality. However, when $Y$ represents the recidivism score generated by an algorithm, it introduces tools for the algorithm to \emph{infer} zip code either through input or estimation, potentially leading to actively discriminating against African-American individuals.

Counterfactual fairness stands as a logically sound concept, yet the consensus on its precise definition remains elusive. In our paper, we opted for the counterfactual family of effects (DE, IE, and SE) proposed by \cite{zhang2018fairness}, deeming it the most appropriate among available options. However, alternative ways to define counterfactual fairness may exist, and as we gain a deeper understanding of counterfactual probabilities and formulate refined definitions, our \Cref{alg.fairness} can be applied to accommodate the new definitions.

\paragraph{\bf Analysis of COMPAS scores}

The qualitative analysis of COMPAS scores requires several caveats. First, the validity of our findings depends on the assumption that the underlying causal structural model is correct. The causal discovery process is influenced by various factors, including the choice of algorithm and the definition of variables; for example, adopting a different causal discovery algorithm could result in slightly different causal diagrams.
While this does not invalidate our results, it emphasizes the need to interpret the findings with the understanding that the underlying causal model may change. To address this, Algorithm \ref{alg.fairness} is designed to account for multiple causal models, providing a more robust framework for analysis.

Another key caveat is that the interpretation of results is closely tied to the definitions of the variables used. In our analysis, race is defined as a binary variable: Black versus Others. As a result, the analysis specifically addresses counterfactual scenarios involving changes in race between Black and non-Black groups.
This framing does not extend to other populations, such as the Hispanic demographic. For instance, if we had defined the Race variable as Hispanic versus Others, the results would have been entirely different. These limitations must be carefully considered when interpreting and generalizing the findings.

%% file: section7.tex
\section{Conclusion}

This paper addressed the non-identifiability issue associated with counterfactual probabilities by introducing a sampling algorithm (\Cref{alg.fairness}) for their estimation.
The simulation study in Section \ref{sec:simulation} affirmed the validity of our algorithm—all results encompass the ground-truth value within the $95\%$ confidence interval.
Then, we demonstrated our algorithm in Section \ref{sec:compas} by evaluating algorithmic fairness of COMPAS recidivism scores with respect to race, age, and sex.
The findings revealed a significant spurious effect (SE) of $25.5\pm 2\%$ when the race changed to African-American from others. Additionally, a negative direct effect (DE) was identified when age transitioned from less than $30$ to over $30$. In the case of sex, the confidence interval was too wide and included zero, not yielding decisive conclusions. This lends credibility to our COMPAS analysis, affirming the reliability of our approach in efficiently assessing algorithmic fairness.

%% file: appendix.tex
\section{Additional Results} \label{app:exp}

In this appendix, we provide more visualizations to better understand our paper. Some of the representative results are already included in the main part, but we show them again in this section for completeness.

\subsection{From Section \ref{sec:simulation}. Simulation Study}\label{app.simu}
In this section, we show more visualizations of the results from the simulation study.

In \Cref{fig.simu.all}, the top left graph shows the convergence of $P[Y_{A=0}=1|A=0]$ computed from the samples. The X axis is time ($1^{st}$ to $4000^{th}$ sample) and the horizontal red line indicates the ground truth value for $P[Y=1|A=0]$. For the three histograms, the black vertical line is the groundtruth value (labeled as TRUE)and the two red lines show 95\% confidence interval (2.5\% top, 2.5\% bottom).

\begin{figure}[ht]
    \centering
    \includegraphics[width=6.5cm]{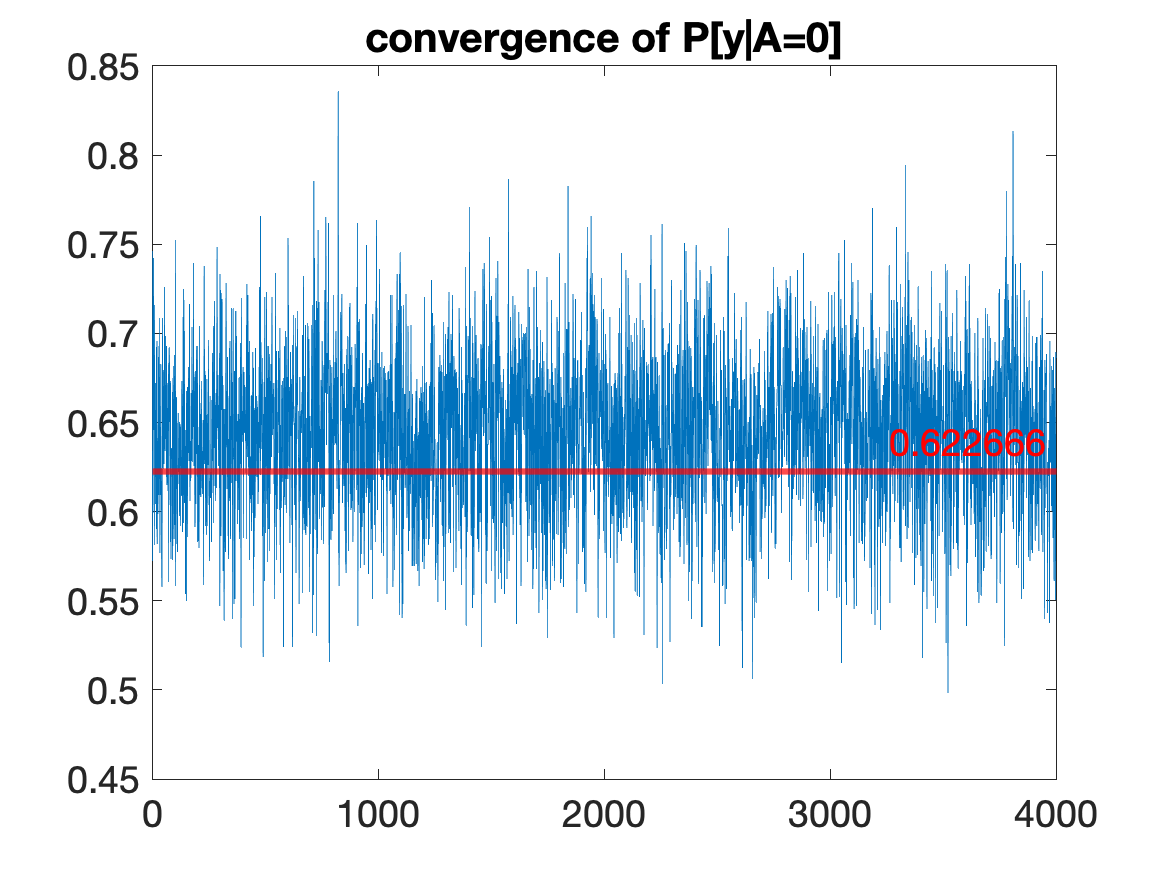}
    \includegraphics[width=6.5cm]{img/simulation/hist-Ctf-DE.png}
    \includegraphics[width=6.5cm]{img/simulation/hist-Ctf-IE.png}
    \includegraphics[width=6.5cm]{img/simulation/hist-Ctf-SE.png}
    \caption{The convergence graph (top left), and histograms for DE (top right), ID (bottom left) and SE (bottom right), obtained from the simulation dataset.}
    \label{fig.simu.all}
\end{figure}

\subsection{From Section \ref{sec:scm}. Causal Graphical Structure of COMPAS}
\label{app.equiv.class}

Figure \ref{fig.equiv.class} show the equivalence class identified from the COMPAS data using the FCI algorithm.
FCI algorithm identifies a equivalence class $\mathcal{E}$ of causal diagrams in a compact form of a partial ancestral graph (PAG), where the circle edges represent undetermined
edge marks \cite{zhang2008completeness}.
To translate the PAG into a DAG, we need to decide the mark for circle edges using qualitative knowledge such as time order among the variables.
More details about how to interpret the equivalence class can be found in Section 3 of \cite{zhang2008completeness}.

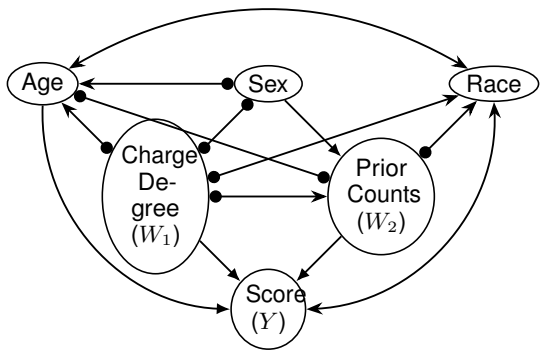
\begin{figure}[ht]
    \centering
    \begin{tikzpicture}
			\def\outerr{3.2}
			\def\innerr{3}

			\node[textnode] (A) at (0, 0) {Age};
			\node[textnode] (S) at (3, 0) {Sex};
			\node[textnode] (R) at (6, 0) {Race};
			\node[textnode, text width=0.9cm, align=center] (W1) at (1.5, -1.5) {Charge Degree ($W_1$)};
			\node[textnode, text width=0.9cm, align=center] (W2) at (4.5, -1.5) {Prior Counts ($W_2$)};
			\node[textnode, text width=0.6cm, align=center] (Y) at (3, -3) {Score ($Y$)};

			\draw[Circle-Circle, line width=0.25mm] (A) -- (W2);
			\draw[dir] (S) -- (W2);
                \draw[Circle-Circle, line width=0.25mm] (S) -- (W1);
                \draw[dir] (A) to [bend right = 45] (Y);
                \draw[dir] (W1) -- (Y);
                \draw[dir] (W2) -- (Y);
			\draw[Stealth-Circle, line width=0.25mm] (A) -- (S);
			\draw[Stealth-Stealth, line width=0.25mm] (A) to [bend left = 30] (R);
			\draw[Stealth-Stealth, line width=0.25mm] (R) to [bend left = 45] (Y);
			\draw[Stealth-Circle, line width=0.25mm] (R) -- (W2);
                \draw[Stealth-Circle, line width=0.25mm] (R) -- (W1);
                \draw[Circle-Stealth, line width=0.25mm] (W1) -- (W2);
			\draw[Stealth-Circle, line width=0.25mm] (A) -- (W1);
    \end{tikzpicture}
    \caption{Identified equivalent class $\mathcal{E}$ for the COMPAS dataset}
    \label{fig.equiv.class}
\end{figure}

Note that categorical variables must be binary for causal discovery; thus, Race is defined as $1$ for African-American and $0$ otherwise. For Age, Prior Counts, and Score, we binarize the variables for computational efficiency, following the rules in Table \ref{t.variables}. However, our sampling algorithm (Algorithm \ref{alg.gibbs}) can accommodate any finite discrete variables.

\subsection{From Section \ref{sec:results}. Estimated Counterfactual DE, IE, and SE}\label{app.compas}
In this section, we show more visualizations of the results from the COMPAS Case Study (Section \ref{sec:results}).

\subsubsection{$A=\mbox{Race}$}
The graph on the left side of \Cref{fig.race.all} shows the time series of a counterfactual probability $P[Y_{A=0}=1|A=0]$ computed from the samples as a blue line. The red line is the same value obtained from the observation (because $P[Y_{A=0}=1|A=0] = P[Y=1|A=0]$.) The values in blue line should converge to the red line, if the sampling method worked out. In this case, the counterfactual probability computed by the \Cref{alg.gibbs} lies between around $0.22$ and $0.25$, and the observational distribution (around $0.2350$) also falls within the obtained bound.

\begin{figure}[ht]
    \centering
    \includegraphics[width=6.5cm]{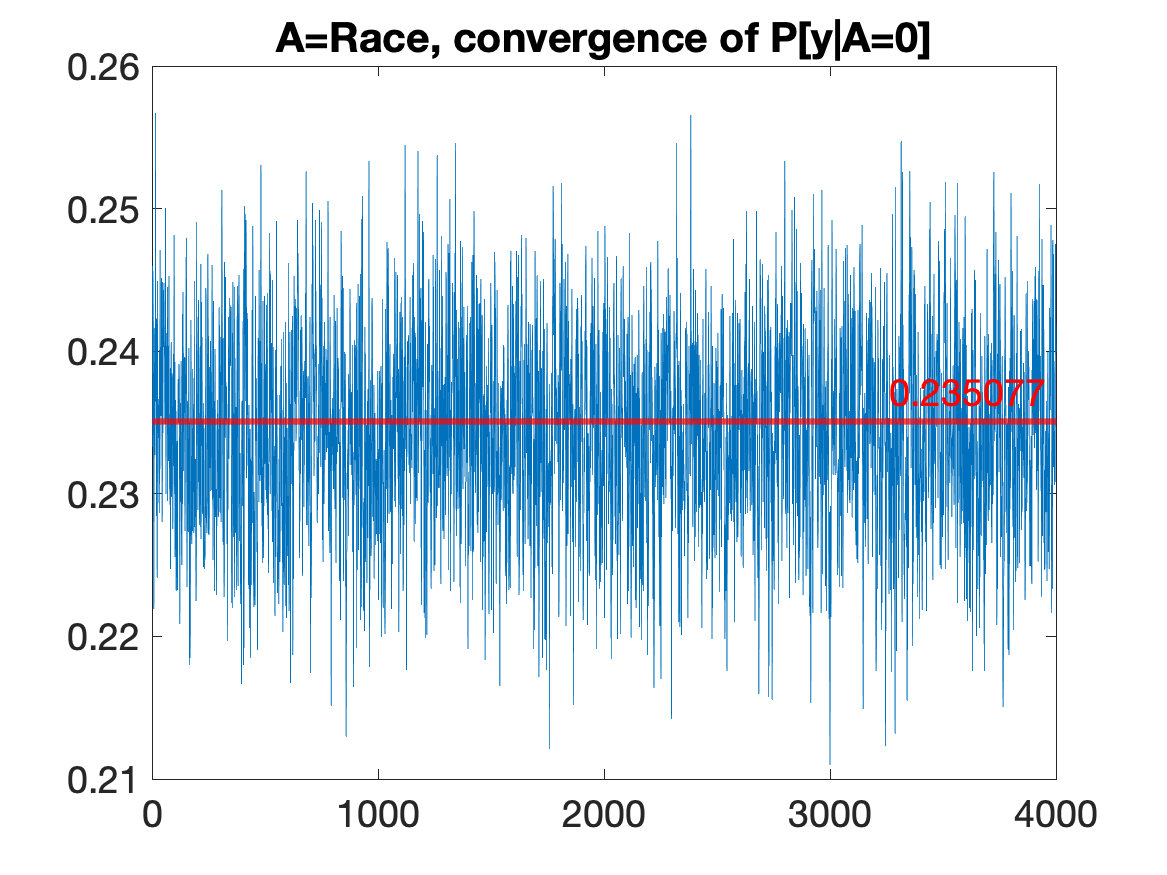}
    \includegraphics[width=6.5cm]{img/compas/Race-hist-Ctf-SE.png}
    \caption{
    \textit{Left:} the convergence of $P[Y_{A=0}=1|A=0]$ computed from the samples. The X axis is time (from first to $4000$-th sample) and the horizontal red line indicates $P[Y=1|A=0]$ computed from the data. \textit{Right:} SE of $A=\mbox{Race}$. The two red lines show 95\% confidence interval (2.5\% top, 2.5\% bottom).}
    \label{fig.race.all}
\end{figure}

\subsubsection{$A=\mbox{Age}$}

\Cref{fig.age.converge} shows the calculated $P(y|A=0)$, $A=\mbox{Age}$, over each round of draw (similar to the graph on the right side of \Cref{fig.race.all}). This is to test the convergence, where we expect the blue lines to converge to the red horizontal line (groud truth from observation). 
Indeed, the blue line and the red line overlaps well: blue lines falling mostly between $0.46$ ad $0.49$, and the red line indicates $0.47$. The variance in the blue line is slightly bigger than the Race case, but still within a reasonable range.

\begin{figure}[ht]
    \centering
    \includegraphics[width=5cm]{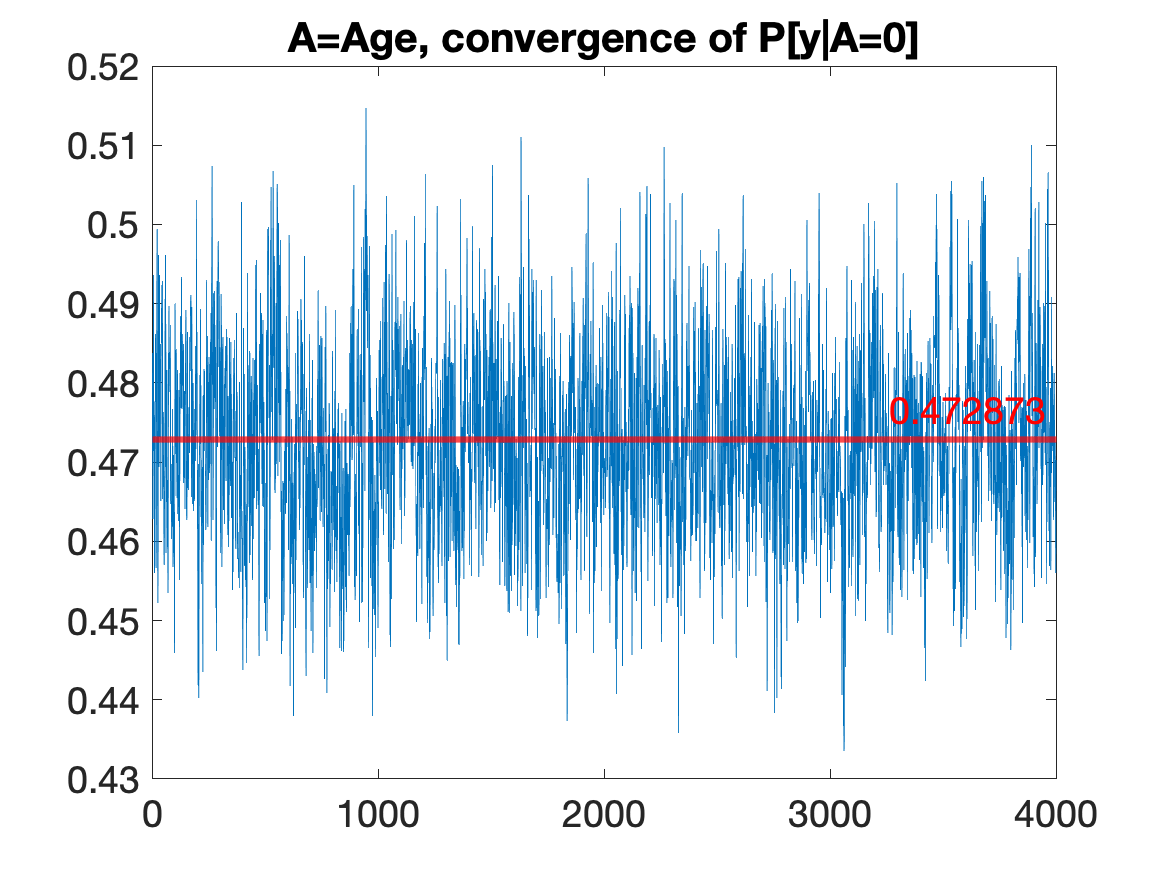}
    \caption{The convergence of $P[Y_{A=0}=1|A=0]$, when $A=\mbox{Age}$, computed from the samples. The X axis is time ($1^{st}$ to $4000^{th}$ sample) and the horizontal red line indicates $P[Y=1|A=0]$ computed from the data.}
    \label{fig.age.converge}
\end{figure}

\Cref{fig.age.hist} shows the histogram of all three counterfactual fairness quantities, when $A$ is set to be Age. As discussed earlier, the width of the confidence interval of SE is $41.68\%p$ and this bound is much wider than the bound for SE of $A=\text{Race}$. From this confidence interval, it is hard to argue if there exists any discrimination, let alone the magnitude of direction of it.

\begin{figure}[ht]
    \centering
    \includegraphics[width=5cm]{img/compas/Age-hist-Ctf-DE.png}
    \includegraphics[width=5cm]{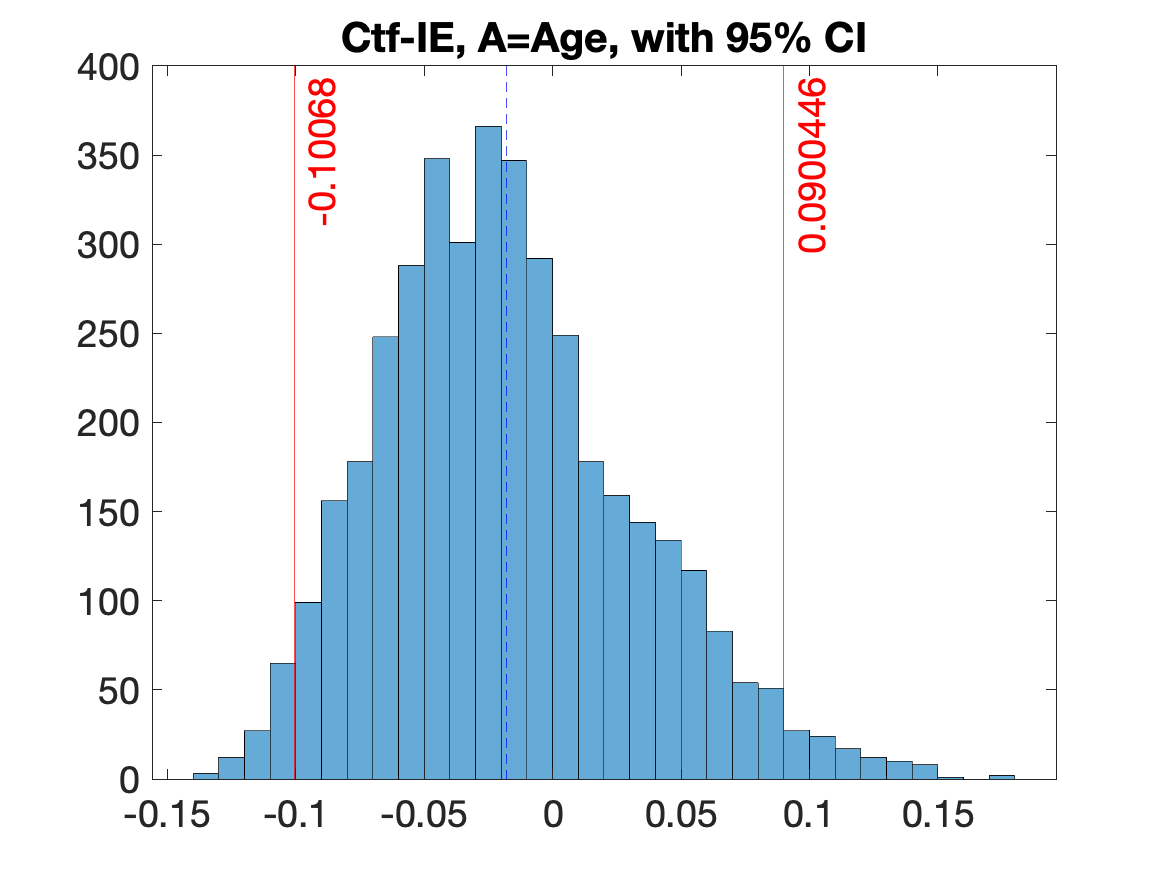}
    \includegraphics[width=5cm]{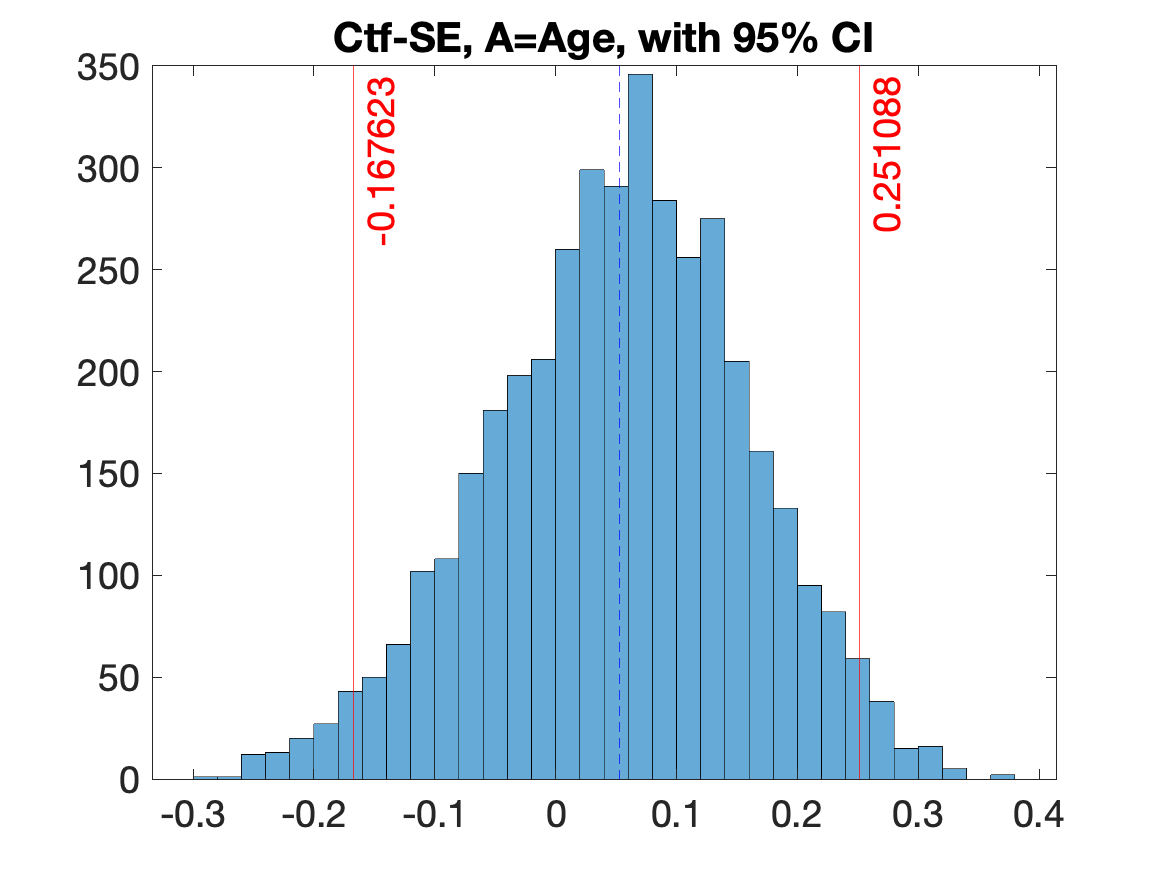}
    \caption{DE (top left), IE (top right), and SE (bottom) of $A=\mbox{Age}$. The blue line is the mean and the two red lines show 95\% confidence interval (2.5\% top, 2.5\% bottom).}
    \label{fig.age.hist}
\end{figure}

\subsubsection{$A=\mbox{Sex}$}

\Cref{fig.sex.converge} shows the largest variance among all three $A$'s we tested. The blue line ranges mostly between $0.3$ and $0.36$, showing a higher variance than the two previous cases. The red line ($0.3146$) lies within that bound, but it is much closer to $0.3$ than $0.36$. This may imply that our \Cref{alg.gibbs} found it harder to converge in the case of $A=\text{Sex}$.

\begin{figure}[ht]
    \centering
    \includegraphics[width=7cm]{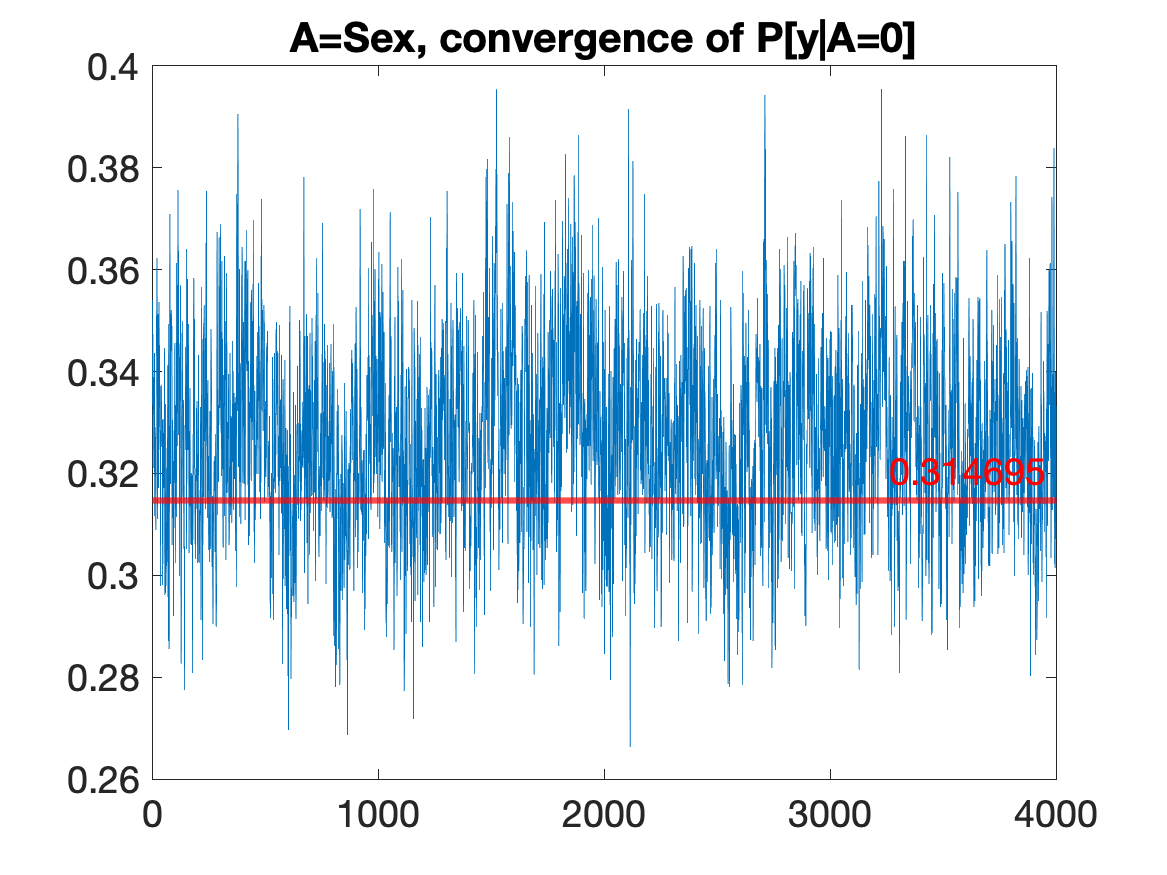}
    \caption{The convergence of $P[Y_{A=0}=1|A=0]$ computed from the samples. The X axis is time (from first to $4000$-th sample) and the horizontal red line indicates $P[Y=1|A=0]$ computed from the data.}
    \label{fig.sex.converge}
\end{figure}

\Cref{fig.sex.hist} shows the distribution of IE (left) and SE (right). Again, the confidence interval of IE is $23.1\%p$ (CI=$(-0.0581,0.1738)$), while that of SE is $28.05\%p$ (CI=$(-0.1284,0.1521)$).
Although it is difficult to decide which direction the spurious effect is heading to, IE seems to be marginally negative or positive.
These confidence intervals are much wider than the previous cases, which is already alluded by the convergence in \Cref{fig.sex.converge}.

\begin{figure}[ht]
    \centering
    \includegraphics[width=6.5cm]{img/compas/Sex-hist-Ctf-IE.png}
    \includegraphics[width=6.5cm]{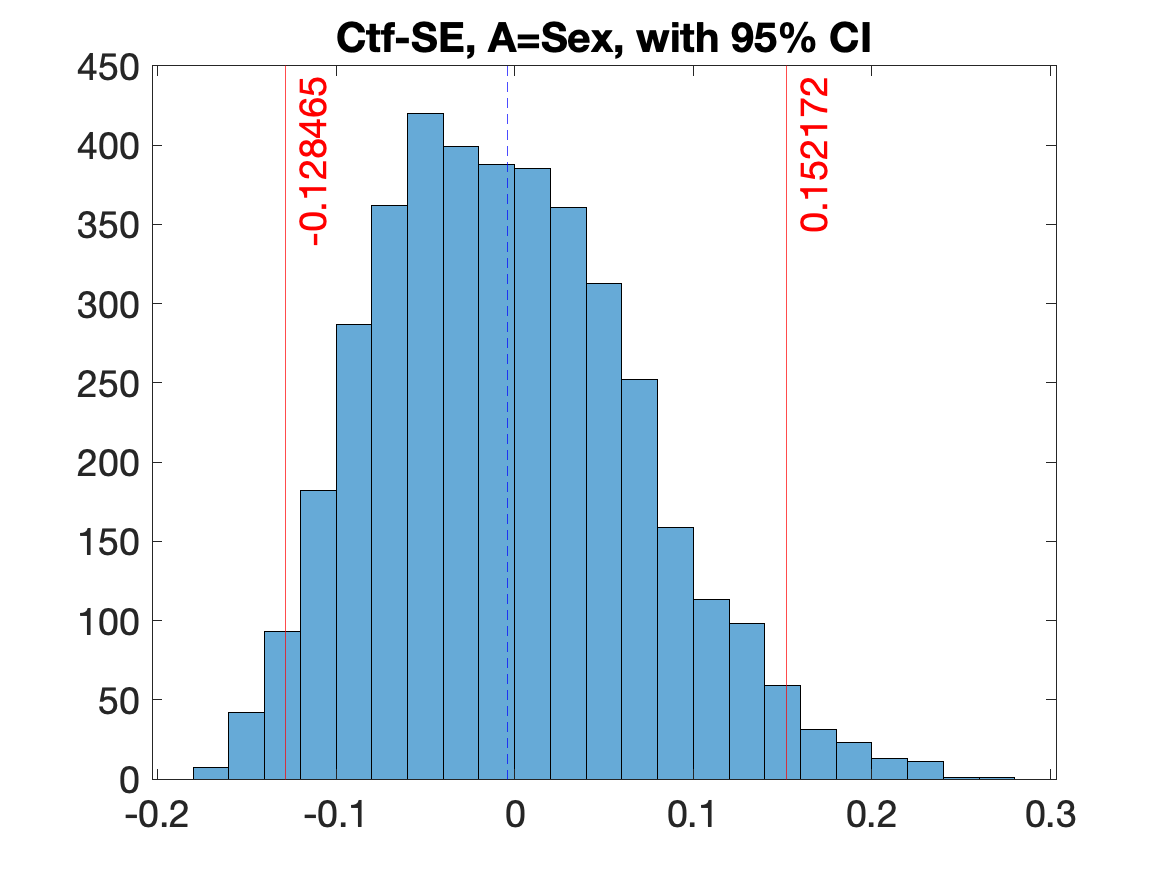}
    \caption{IE (left) and SE (right) of $A=\mbox{Sex}$. The blue line is the mean and the two red lines show 95\% confidence interval (2.5\% top, 2.5\% bottom).}
    \label{fig.sex.hist}
\end{figure}

\newpage

\subsection{From Section \ref{sec:comparison}. Comparison to Other Fairness Measures}\label{app.othermeasure}
In this section, we show more visualizations of the results from bounding other fairness measures.
All results are already presented and discussed in Section \ref{sec:comparison}.

\begin{figure}[ht]
    \centering
    \includegraphics[width=6.5cm]{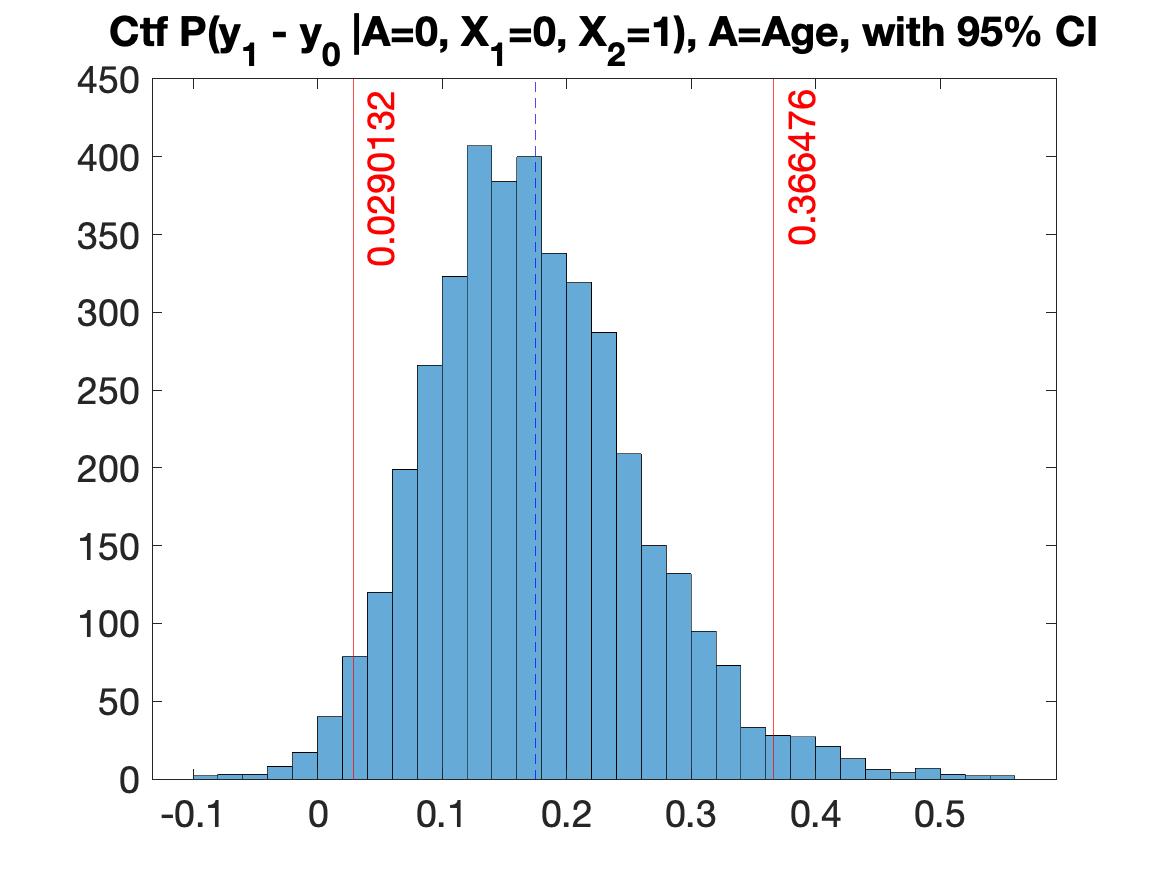}
    \includegraphics[width=6.5cm]{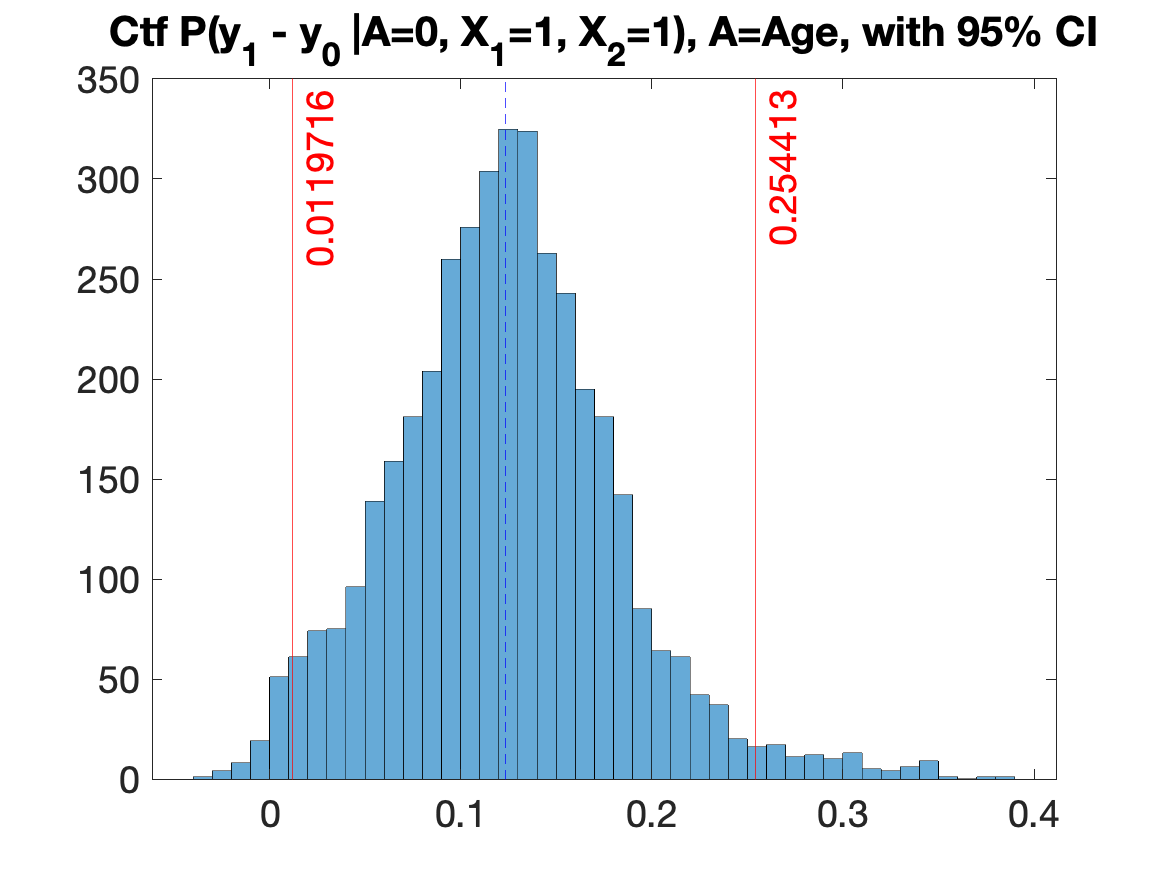}
    \caption{CE with $A=\mbox{Age}$ conditioned on $A=0, W_1=0, W_2=1$ (left) and $A=0, W_1=1, W_2=1$ (right). The two vertical red lines show 95\% confidence interval (2.5\% top, 2.5\% bottom).}
    \label{fig.age.ce}
\end{figure}

\begin{figure}[ht]
    \centering
    \includegraphics[width=6.5cm]{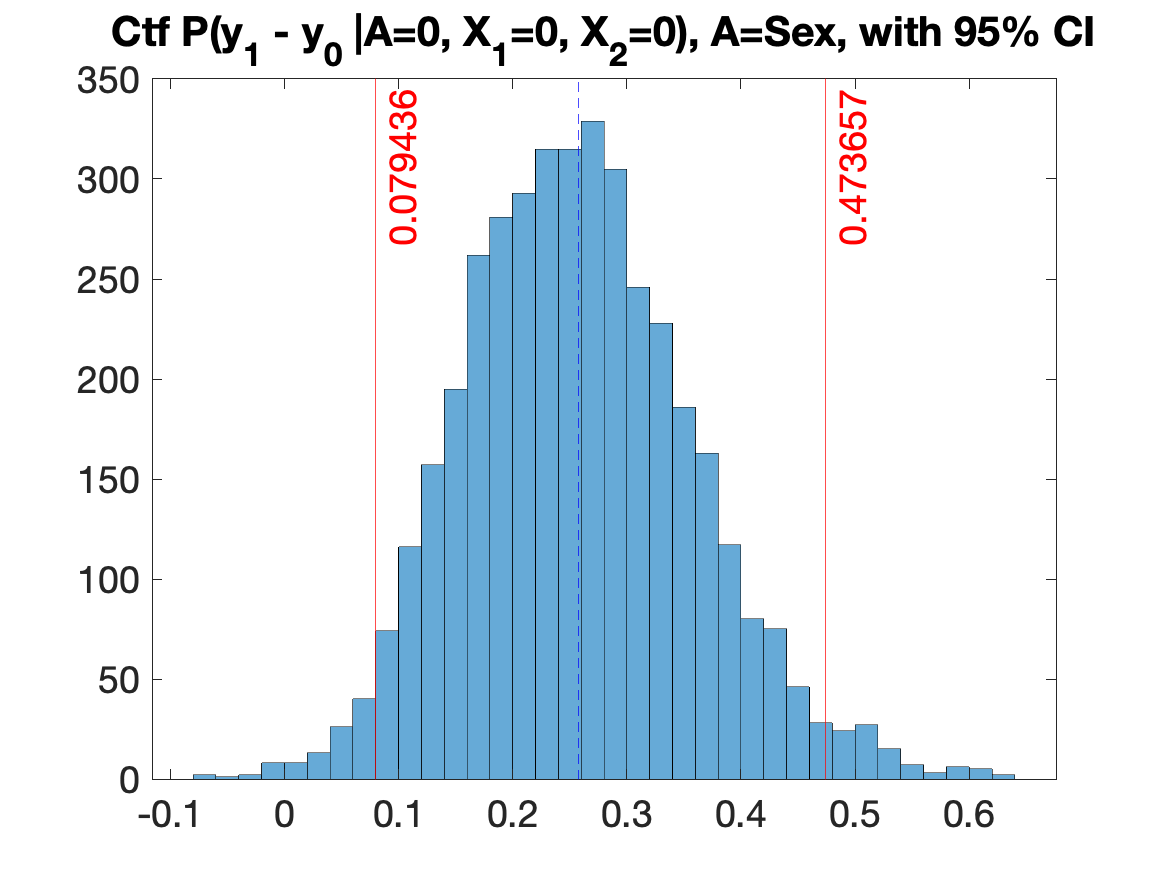}
    \includegraphics[width=6.5cm]{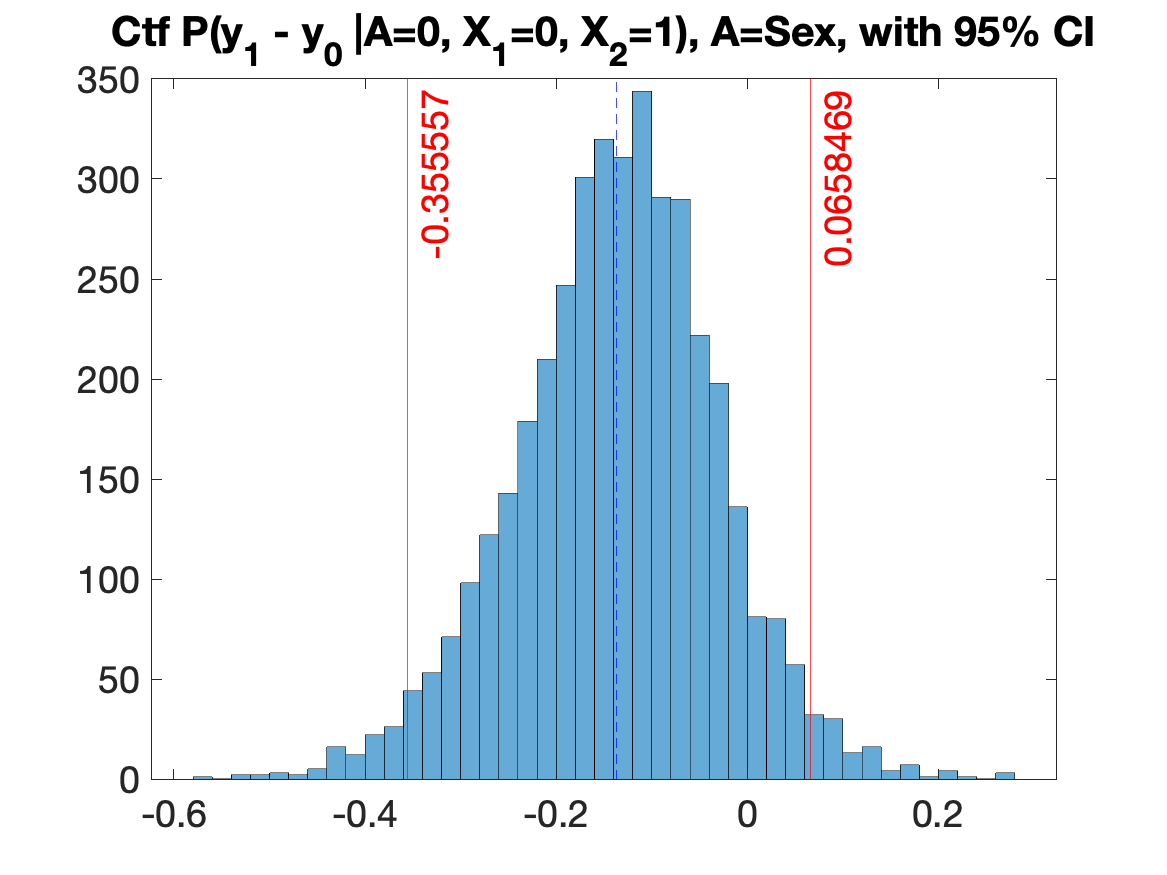}
    \includegraphics[width=6.5cm]{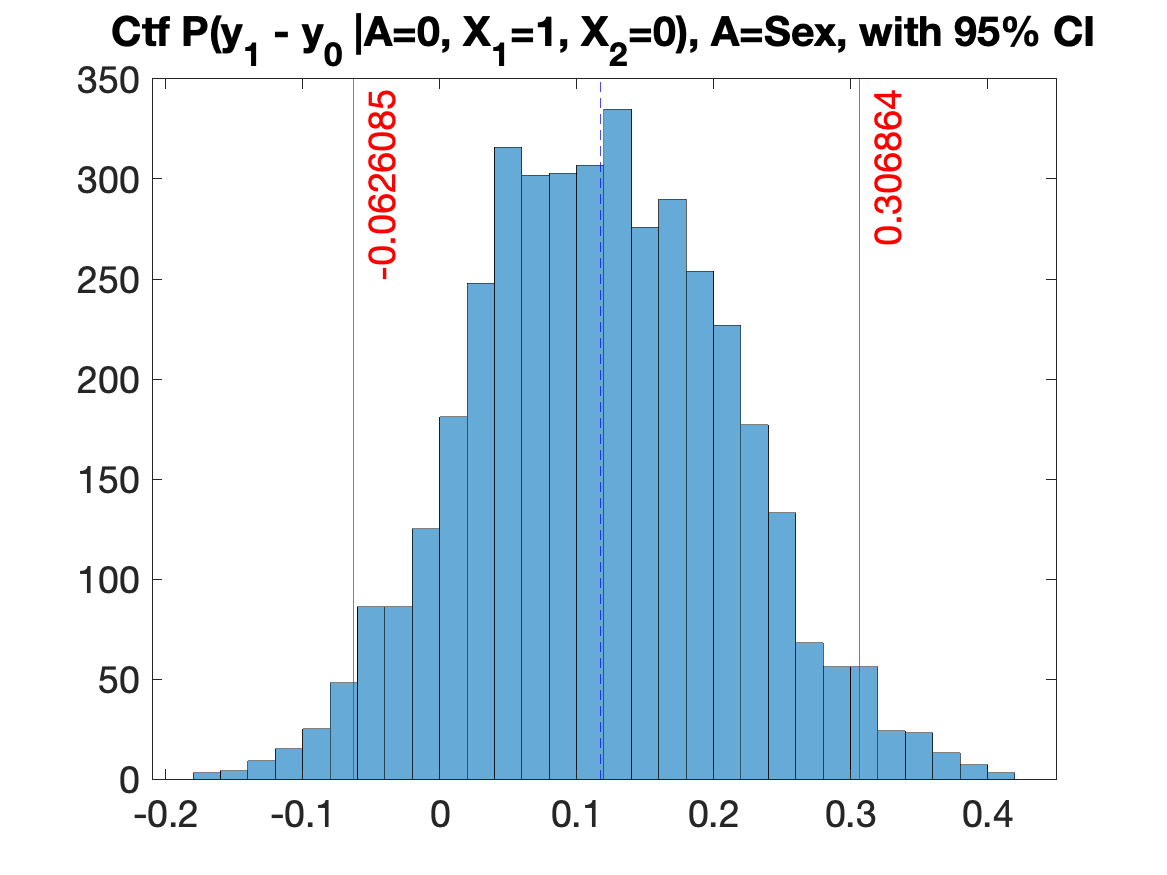}
    \includegraphics[width=6.5cm]{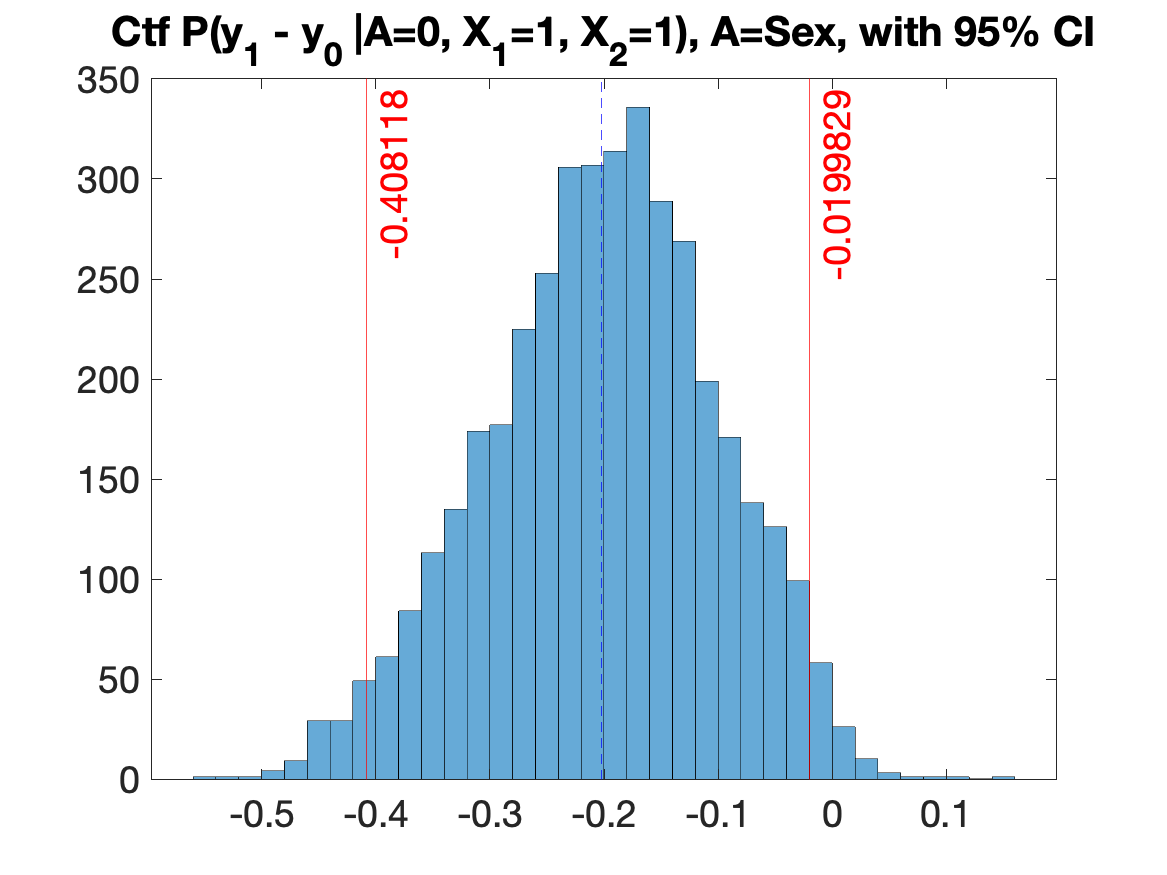}
    \caption{CE with $A=\mbox{Sex}$ conditioned on $A=0$ and different values of $W_1$ and $W_2$}
    \label{fig.sex.ctf}
\end{figure}

\begin{figure}[ht]
    \centering
    \includegraphics[width=6.5cm]{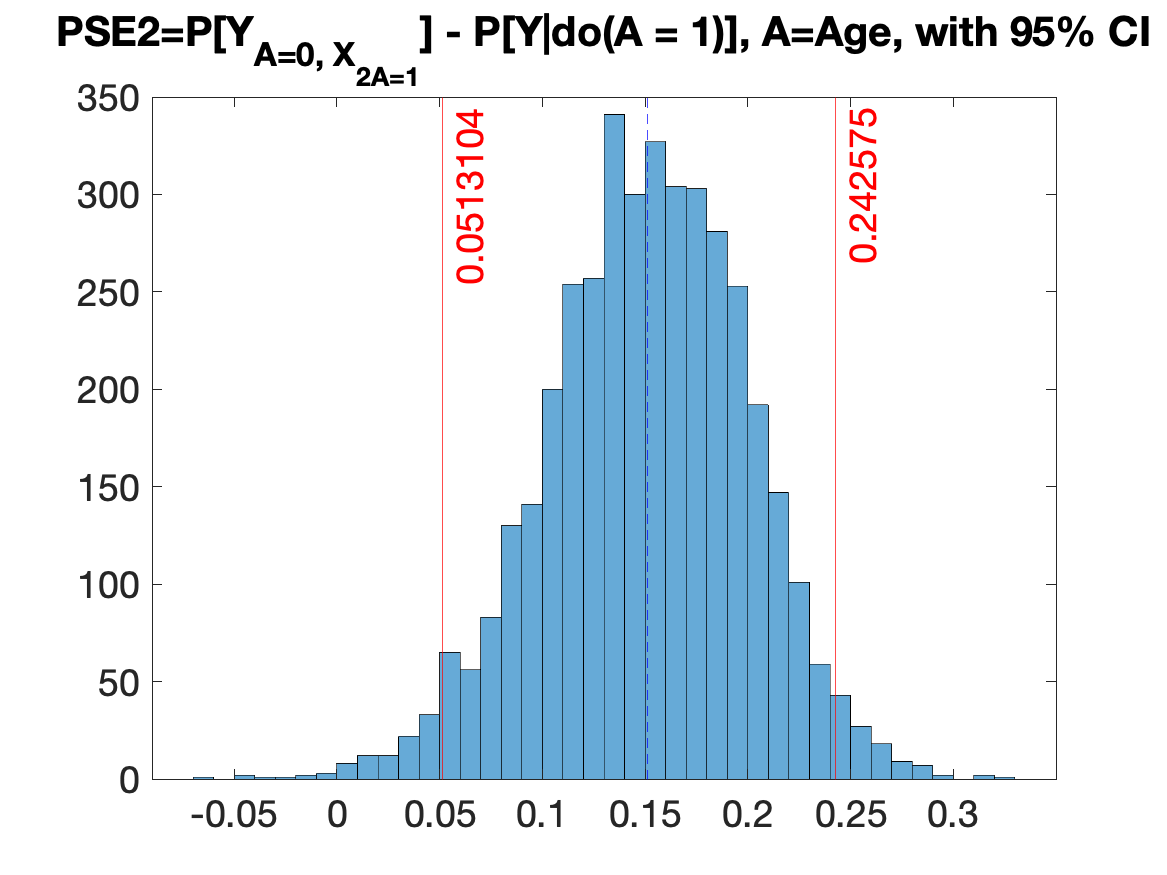}
    \includegraphics[width=6.5cm]{img_ctf/Age-PSE-PSE1.png}
    \caption{PSE of $A=\mbox{Age}$ for the path $A \to Y$ (left) and the path $A \to W_2 \to Y$ (right).}
    \label{fig.age.pse}
\end{figure}

\newpage

\section{Frequently Asked Questions}

In this section, we provide answers to basic questions for the readers not too familiar with causal inference.

\paragraph{Q. What does it mean for counterfactual probabilities to be unidentifiable, and how prevalent is the case?}

A. An unidentifiable problem means that the target quantity is underdefined from the available observed data and assumptions on the causal strctural model. Usually, counterfactual probabilities cannot be specified even if we know the true causal diagram with the data \cite{pearl:2k}. More detailed explanations can be found in Chapters 1 and 7 of \cite{pearl:2k}.

\paragraph{Q. Can we use this algorithm without knowledge of the true data-generating process?}

A. Yes. Our Algorithms only require that the observed variables are discrete and finite, and the equivalence class is obtained by FCI. Section \ref{sec:simulation} uses a simulated dataset generated from a known data-generating process to test the accuracy of our method. However, in real-world cases, as demonstrated in Section \ref{sec:compas}, inference will be based solely on the observed data and detailed knowledge of the true data-generating process is not required.

\paragraph{Q. Then, are we assuming that the cardinality of the variables are bounded?}

A. No. The bounded cardinality is not an assumption, but a sufficient statistic that allows us to represent the observational distribution and the target counterfactual measurement in the ground-truth causal model. This result was first introduced in \cite{zhang2022partial}, and Theorem \ref{thm:bound} extends it to nested counterfactual fairness measures, while improving the latent cardinality. In essence, Theorem \ref{thm:bound} shows that there exists a natural parametrization for the latent space as a function of the cardinality of the observation distribution, which helps a user to make a data-driven decision.

\paragraph{Q. What if there are multiple candidates for $\mathcal{G}$ in $\mathcal{E}^*$?}

A. Then, the probability mass is equally distributed across all candidates. Its effect on the final bound will depend on the causal diagrams and the data, but the width of the bound is expected to increase when more options are considered (models with fewer assumptions).